%% file: main_paper.tex
\documentclass[twoside]{article}

\usepackage[accepted]{aistats2026}
\usepackage{natbib} 

\usepackage[utf8]{inputenc} 
\usepackage[T1]{fontenc}    
\usepackage{hyperref}       
\usepackage{url}            
\usepackage{booktabs}       
\usepackage{amsfonts}       
\usepackage{nicefrac}       
\usepackage{microtype}      
\usepackage{xcolor}         
\usepackage{amsthm}         
\usepackage{amssymb}        
\usepackage{graphicx}
\usepackage{enumitem}       
\usepackage{nicematrix} 
\usepackage{subcaption} 
\usepackage{multirow}   
\usepackage{dblfloatfix} 
\graphicspath{ {./Figures/} }
\input{macros}

\usepackage{wrapfig}

\begin{document}
\runningauthor{Maude Lizaire, Michael Rizvi-Martel, Éric Dupuis  Guillaume Rabusseau}

\twocolumn[

\aistatstitle{On the Role of Depth in the Expressivity of RNNs}

\aistatsauthor{Maude Lizaire \And Michael Rizvi-Martel \And Éric Dupuis \And Guillaume Rabusseau }

\aistatsaddress{ Mila \& DIRO \\ Université de Montréal \And  Mila \& DIRO \\ Université de Montréal \And Independent \And Mila \& DIRO, CIFAR AI Chair \\ Université de Montréal} ]

\begin{abstract}
The benefits of depth in feedforward neural networks are well known: composing multiple layers of linear transformations with nonlinear activations enables complex computations. While similar effects are expected in recurrent neural networks (RNNs), it remains unclear how depth interacts with recurrence to shape expressive power. Here, we formally show that depth increases RNNs’ memory capacity efficiently with respect to the number of parameters, thus enhancing expressivity both by enabling more complex input transformations and improving the retention of past information. We broaden our analysis to 2RNNs, a generalization of RNNs with multiplicative interactions between inputs and hidden states. Unlike RNNs, which remain linear without nonlinear activations, 2RNNs perform polynomial transformations whose maximal degree grows with depth. We further show that multiplicative interactions cannot, in general, be replaced by layerwise nonlinearities. Finally, we validate these insights empirically on synthetic and real-world tasks.
\end{abstract}

\begin{figure}[h]
\centering
\includegraphics[width=0.47\textwidth]{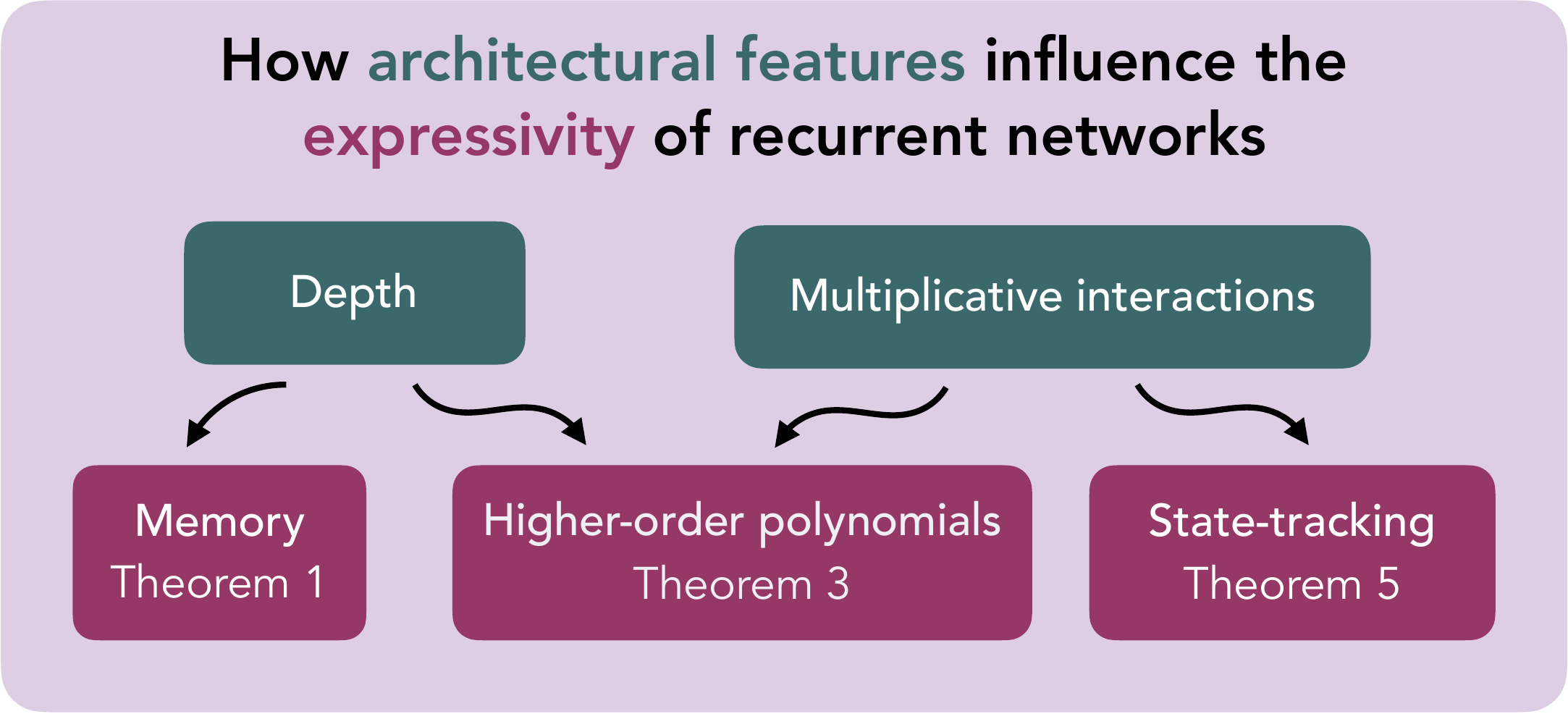}
\caption{Theoretical insights overview of the effects of architectural choices on the expressivity of RNNs.}\label{fig:summary}
\end{figure}

\section{Introduction}
It is well known that Feedforward neural networks (FNNs) are universal approximators~\citep{hornik1989multilayer,cybenko1989approximation}. In practice however, their capacity is limited as they may require impractically large hidden sizes to compute complex functions. That is where the benefit of depth comes to play. Indeed, stacking multiple layers results in a composition of nonlinear transformations, and it is understood that increasing the number of such compositions allows the network to represent more complex functions~\citep{telgarsky2016benefits,rolnick1705power,montufar2014number,haastad1991power}. In the absence of nonlinearities, however, deep feedforward neural networks essentially collapse into shallow ones, as the composition of linear functions results in just another linear transformation.

A similar situation occurs in neural networks computing functions over sequences. Recurrent neural networks~(RNNs) are known to be Turing complete~\citep{siegelmann1992computational,siegelmann1994analog,siegelmann1996recurrent}, but their remarkable expressive power relies on impracticable assumptions, namely infinite precision or unbounded computation time~\citep{weiss2018practical}. The impact of depth in recurrent networks is however less straightforward than in feedforward networks.
While the benefits of depth in FNNs can be extended to RNNs by simply ignoring recurrence, this overlooks the complex dynamics introduced by recurrent connections.
Indeed, in sequence modeling, the expressivity of a model involves not only to transform inputs in meaningful representations, but also the capacity to propagate and combine information through time in useful ways. In order to fundamentally understand how depth influences the expressive power of RNNs over sequences, the subtle interplay between recurrence and depth ought to be examined.

In this work, we investigate how depth influences the expressivity of RNNs. First, we focus on \emph{linear} RNNs to isolate the interplay between recurrence and depth from the expressive gain arising by composing nonlinear activations. We formally show that deep linear RNNs are strictly more expressive than shallow ones as they have a greater ability to memorize information. Furthermore, we prove that increasing depth, rather than hidden size, is a more parameter-efficient approach to enhance the network’s memory.

Second, we explore how the effect of depth manifests in models with multiplicative interactions between inputs and hidden states, which we refer to as second-order recurrent neural networks~(2RNNs). Here, stacking layers has a similar effect as the composition of nonlinear activations: it expands the class of functions the model can represent. Specifically, linear 2RNNs compute polynomials of their inputs, whose degree increases with the number of layers. 
As a result, deep linear 2RNNs are strictly more expressive than their shallow counterparts. We further consider models whose bilinear terms are parameterized by a CP decomposition, CPRNNs, and show that depth does not alter the expressive gain obtained by increasing the rank of the decomposition.

Third, we investigate how the gain in expressive power from stacking nonlinear activations differs from the one provided by multiplicative interactions, showing that there exist functions computable by single-layer 2RNNs (specifically those requiring state-tracking) that cannot be realized by deep RNNs with nonlinear activations applied only depth-wise.

We study how these theoretical findings translate in practice with gradient descent optimization through synthetic and real data experiments on RNNs, 2RNNS and SSMs~(S4). Empirically, RNNs capacity to memorize and copy information with respect to depth supports our theoretical analysis, even when nonlinearity is added. When tested on parity, a task not requiring memory, but rather the ability to state-track via temporal multiplicative interactions, and found that the impact of depth is highly dependent on the way nonlinearities are applied~(i.e. recurrently or only in depth). On real datasets, whether performing language modeling on tiny Shakespeare or testing Long Rang Arena benchmarks,  we find that depth improves performance quite consistently, while the parameter efficiency benefit is task-dependent as predicted theoretically. 

Our contributions can be summarized as follows:  
\begin{itemize}
    \item Our theoretical analysis reveals that even in the absence of nonlinearities, depth strictly increases the expressivity of RNNs (Theorem~\ref{thm:deepRNN}). More precisely, it enlarges the hidden capacity of the network, and we show that for certain tasks, such as those requiring memory, depth provides a parameter-efficient means to improve expressivity (Theorem~\ref{thm:params}).
    \item Whereas linear RNNs compute only linear transformations of their inputs regardless of depth, we show that in 2RNNs, increasing the number of layers directly enables the computation of higher-order polynomials (Theorem~\ref{thm:BIRNN}). This result extends to CPRNNs, and we further show that depth does not alter the effect of the rank on the network’s capacity (Theorem~\ref{thm:cprnn}). 
    \item We establish a fundamental separation between the effects of nonlinear activations and multiplicative interactions, by showing that certain functions computed by single-layer 2RNNs cannot be realized by deep RNNs with nonlinear activations only in depth (Theorem~\ref{thm:WFA}).
    \item Finally, we provide a diverse set of experiments that validates and illustrates the insights revealed by our theoretical analysis\footnote{Code base for this paper can be found at \url{https://github.com/MaudeLiz/Role_of_Depth_in_RNNs}}. 

\end{itemize}

\paragraph{Related work}
RNNs~\citep{elman1990finding} are a natural choice for sequence modeling. Many gated variants of such models were introduced to mitigate vanishing/exploding gradient issues~\citep{hochreiter1997long,cho2014properties,chung2014empirical}. Recently, the advent of SSMs has brought recurrent models back into the limelight~\citep{gu2021efficiently,gu2023mamba,gu2020hippo,nguyen2022s4nd,gupta2022diagonal}. This renewed interest has prompted recent studies on the expressivity and limitations of SSMs~\citep{wang2023state,wang2024understanding,grazzi274141450unlocking}. In particular, it was shown that these models, who only have nonlinear activations depth-wise cannot state-track~\citep{merrill2024illusion}.
Early research on expressive power of sequence models~\citep{siegelmann1992computational,steijvers2019recurrent,boden1999learning} showed that recurrent models can recognize regular and other formal languages and are even Turing complete (with unbounded precision and compute time). 
More recently, novel hierarchies have been used to analyze the expressivity of RNNs~\citep{merrill2020formal} and theoretical work have highlighted fundamental differences between RNNs and transformers~\citep{bhattamishra2024separations}.
Previous work has largely overlooked the role of depth in RNNs, which is the focus of this study. In deep FNNs, depth has been shown to provide exponential gains in efficiency for certain functions~\citep{eldan2016power, telgarsky2015representation}. Similarly, depth in CNNs has been extensively studied using tensor network analysis~\citep{cohen2016convolutional, cohen2016expressive, cohen2016inductive, sharir2017expressive, alexander2023makes, razin2024ability}. 
For RNNs, analysis based on surrogate tensor decomposition  models demonstrated the benefit of depth, even for tasks that do not share a sequential inductive bias~\citep{khrulkov2017expressive,khrulkov2019generalized,levine2018benefits}.
Finally, second-order interactions have been used to make deep learning models both more interpretable~\citep{pearce2024bilinear} and more expressive~\citep{Irsoy2014ModelingCW, Jayakumar2020Multiplicative,cheng2024multilinear}. This led to the introduction of various multiplicative RNN architectures~\citep{tjandra2016gated, krause2017,sutskever2011generating,wu2016multiplicative,su2024language}. The expressivity of 2RNNs and their variants has been the object of theoretical studies~\citep{li2022connecting,lizaire2024tensor}.

\section{Preliminaries}

We begin by introducing the notation used in this work as well as the models studied. 
\vspace{-0.1cm}
\paragraph{Notation}
Our notation conventions are as follows: vectors, matrices and higher-order tensors are respectively denoted by bold lowercase letters $\v \in \R^d$, bold uppercase letters $\M \in \R^{d_1 \times d_2}$, and bold calligraphic letters $\Tt \in \R^{d_1 \times d_2 \times \dots \times d_p}$. A diagonal matrix with vector $\vb$ on the diagonal is noted $\diag(\vb)$.
The $n$-mode product of a tensor with a vector\footnote{Note that this notation  differs  from the one in the reference~\citep{kolda2009tensor} where $\times_n$ denotes the $n$-mode product of a tensor with a \textit{matrix}.} is defined by: $(\Tt~\times_n~\vec{v})_{i_1,\dots,i_{n-1},i_{n+1},\dots,i_{p}}~=~\sum_{i_n=1}^{d_n} \Tt_{i_1,\dots,i_{p}}\;\vec{v}_{i_n}$. Lastly, $[n]$ denotes the set of integers from $1$ to $n$ and $\ceil{x}$ is the smallest integer greater or equal to $x$.
\vspace{-0.1cm}
\paragraph{Models}
We present formal definitions of RNNs and 2RNNs, then we introduce CPRNNs and BIRNNs based on how they relate to the definitions provided.

\begin{definition}[RNN]\label{def:RNN}
A Recurrent Neural Network of depth $L$ and hidden size $n$ is parameterized by initial hidden state vectors $\vec{h}^{(l)}_0 \in \Rbb^n$,  weight matrices $\Ub^{(l)}, \Vb^{(l)}  \in \Rbb^{n \times n}$ (except $\Ub^{(1)} \in \Rbb^{n \times d}$), bias terms $\bb^{(l)} \in \Rbb^n$ and activation functions $\sigma^{(l)}:\Rbb^n \rightarrow \Rbb^n$ for $l\in[L]$. For any sequence length $T$, an RNN maps a sequence of inputs $\xb_1, \dots, \xb_T\in\Rbb^d$, to a sequence of hidden states $\hb^{(L)}_1, \dots, \hb^{(L)}_T\in\Rbb^n$ via the following computation at each time step $t\in[T]$ and layer $l\in[L]$ with $\hb^{(0)}_t=\xb_t \ \text{ for all } t$: 
\begin{equation}\label{eq:RNN}
\hb_t^{(l)}=\sigma^{(l)}( \Vb^{(l)} \hb_{t-1}^{(l)} + \Ub^{(l)} \hb_t^{(l-1)} + \bb^{(l)} ).
\vspace{-0.1cm}
\end{equation}
\end{definition}

The activation functions are applied element-wise. They generally consist of nonlinear functions such as $\tanh$ or $\text{ReLU}$. In this work, we focus on linear architectures, where $\sigma^{(l)}$ is the identity for all $l$. We call these \emph{linear RNNs}. Section~\ref{sec:theoryWFA} will also study RNNs with nonlinear activations applied only in depth:
\begin{equation}\label{eq:RNNonlydepth}
\hb_t^{(l)}=\Vb^{(l)} \hb_{t-1}^{(l)} + \Ub^{(l)}\sigma^{(l-1)} \left( \hb_t^{(l-1)} \right) + \bb^{(l)}.
\vspace{-0.1cm}
\end{equation}
Second-order RNNs are obtained by adding a bilinear term between inputs and hidden states to Definition~\ref{eq:RNN}. 

\begin{definition}[2RNN]\label{def:2RNN}
A Second-order Recurrent Neural Network of depth $L$ and hidden size $n$ is parameterized by initial hidden state vectors $\vec{h}^{(l)}_0 \in \Rbb^n$,  weight tensors $\At^{(l)} \in \Rbb^{n \times n\times n}$ (except for $\At^{(1)} \in \Rbb^{n \times d\times n}$), weight matrices $\Ub^{(l)}, \Vb^{(l)}  \in \Rbb^{n \times n}$ (except $\Ub^{(1)} \in \Rbb^{n \times d}$), bias terms $\bb^{(l)} \in \Rbb^n$ and activation functions $\sigma^{(l)}:\Rbb^n \rightarrow \Rbb^n$ for $l\in[L]$. For any sequence length $T$, a 2RNN maps a sequence of inputs $\xb_1, \dots, \xb_T$ to a sequence of hidden states $\hb^{(L)}_1, \dots, \hb^{(L)}_T$ via the following computation at each time step $t\in[T]$ and layer $l\in[L]$:
\begin{equation*}\sigma^{(l)}(\At^{(l)} \times_1 \hb_{t-1}^{(l)}\times_2 \hb_t^{(l-1)} + \Vb^{(l)} \hb_{t-1}^{(l)} + \Ub^{(l)} \hb_t^{(l-1)} + \bb^{(l)} ).
\vspace{-0.1cm}
\end{equation*}
\end{definition}

We also consider CPRNNs, which provide a practical alternative to 2RNNs whose weight tensors $\At^{(l)}$ scale cubically. It consists in parameterizing the bilinear term using a CP~decomposition instead of the full third-order tensor. Thus, $\At^{(l)}$ is replaced by a sum of outer products $ \sum_{r=1}^R \vec{a}^{(l)}_r \circ \vec{b}^{(l)}_r \circ \vec{c}^{(l)}_r$ with $\ab_r^{(l)},\bb_r^{(l)},\cb_r^{(l)} \in \Rbb^n$ for all $l\in[L]$~(expect $\bb_r^{(1)} \in \Rbb^d$). The rank $R$ is an additional hyperparameter of the model. For a comprehensive formal definition of this architecture, we refer the reader to~\citep{lizaire2024tensor}.  
Finally, we call BIRNNs second-order RNNs with only a bilinear term (i.e. $\Ub^{(l)}, \Vb^{(l)} \text{ and } \bb^{(l)}$ are all null). Similarly, CPRNNs with only their multiplicative term are CPBIRNNs, and we use the notation CP(BI)RNNs when both CPRNNs and CPBIRNNs are referenced at the same time. 
Figure~\ref{fig:unrolledRNN} presents how recurrent architectures unroll across depth and time. 
It also illustrates that in linear RNNs, the overall computation reduces to a linear map of the inputs, whereas in second-order RNNs (2RNNs), it corresponds to a polynomial function which will be discussed in Sections~\ref{sec:theoryRNNs} and~\ref{sec:2rnn}. 

\begin{figure}[h!]
\centering
\includegraphics[width=0.4\textwidth]{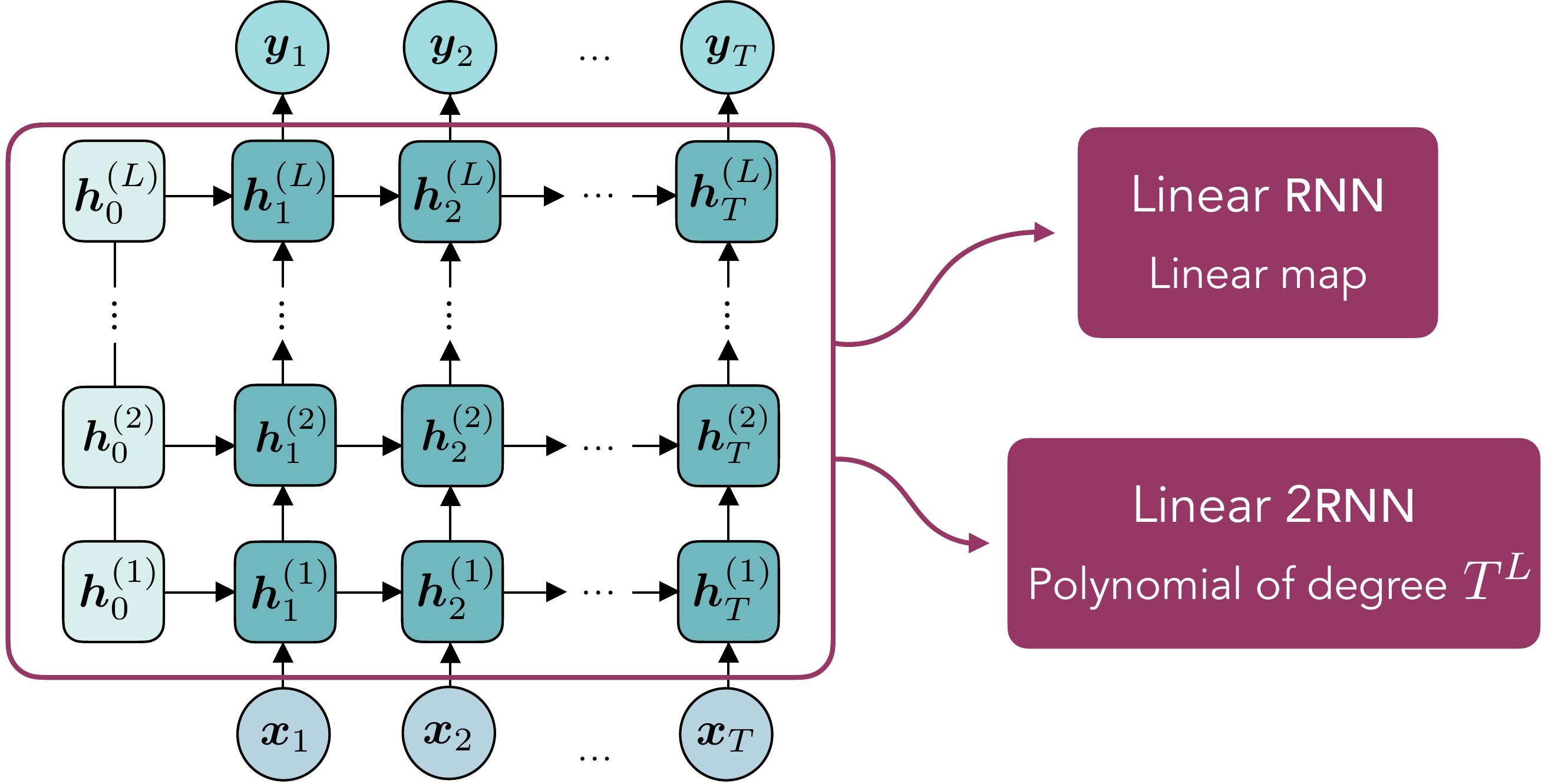}
\caption{Unrolled deep recurrent architecture. Linear RNNs compute linear mappings of the inputs, while linear 2RNNs produce polynomial ones.}\label{fig:unrolledRNN}
\end{figure}

\section{Theoretical Results}
To analyze the effect of depth on the expressivity of RNNs, we introduce a formal definition of the set of functions they can represent. 
\begin{definition}\label{def:functionclass} For any $n\geq 1$ and $L\geq 1$, $\hRNN{n,L}$ denotes the set of functions $h$ mapping input sequences of arbitrary length to the corresponding hidden state sequences computed by RNNs of hidden size $n$ and depth $L$, as given by Def.~\ref{def:RNN}: $h(\vec{x}_1, \vec{x}_2,..., \vec{x}_T) = (\vec{h}_1^{(L)}, \vec{h}_2^{(L)},..., \vec{h}_T^{(L)}).$

\end{definition}
Similarly we define $\hSORNN{n,L}$ for 2RNNs based on Def.~\ref{def:2RNN} and $\hCPRNN{n,L,R}$ for CPRNNs of rank $R$.

\subsection{Role of Depth in (linear) RNNs: Increasing Memory Efficiently}\label{sec:theo-memory}\label{sec:theoryRNNs}
In this section, we focus on linear RNNs, since RNNs with nonlinear activations naturally inherit the same benefits from composing nonlinearities as FNNs.
We begin our analysis by observing that linear RNNs perform \emph{linear} transformations of their inputs, regardless of the network's depth~(See Appendix~\ref{apx:proplinearRNN} for formalization of this statement). Although this observation is reminiscent of the equivalence between deep and shallow linear FNNs, for linear RNNs, there is no collapse of the structure to a single-layer network. In fact, the following theorem states that depth makes linear RNNs strictly more expressive. More precisely, even though linear RNNs compute linear transformations, adding a layer will always make the model strictly more expressive, i.e. able to compute~(linear) functions that could not be computed with one layer less.

\begin{theorem}\label{thm:deepRNN} For any $n>1$ and $L\geq 1$, $\hRNN{n,L} \subsetneq \hRNN{n,L+1}$ for linear RNNs.
\end{theorem}
It is worth taking a moment to contemplate why this result is non-trivial and surprising. Indeed, composing linear functions does not increase expressiveness; the composition remains linear. It is the sequential nature of the model which leads to the strict inclusion in Theorem~\ref{thm:deepRNN}. Intuitively,  adding a layer to an RNN increases its memory capacity, and the network can thus retain information longer.
Although only the hidden vector of the last layer is outputted, the network maintains $L$ hidden vectors, allowing the information to be staged and propagated towards the deeper layers at later time steps.

\paragraph{Sketch of proof}
Proving the inclusion simply consists in finding an explicit parameterization for a model of depth $L+1$ to reproduce the computation of depth $L$ RNN. 
To show strict inclusion, i.e. $\hRNN{n,L} \not \supset \hRNN{n,L+1}$, we look at the capacity of a linear RNN to memorize and propagate information. We thus introduce the function $f_p$ that copies the first component of the input $\xb_t$ $p$ steps forward:
\begin{equation}\label{eq:fp}
    f_p(\xb_1,\cdots, \xb_t) = \begin{cases}
    (\xb_{t-p})_1&\text{ if } t-p > 0\\
    0&\text{ otherwise.}
\end{cases}
\end{equation}
First, we give an explicit construction~(illustrated in Figure~\ref{fig:sketch_constrution}) proving that for any $p$ and any $n$, there exists an RNN of hidden size $n$ and depth $L=\ceil{p/(n-1)}$ that computes $f_p$. Then, looking at the RNN as a flow graph in which information is propagated, we derive an upper bound on the value of $p$ for which an RNN of hidden dimension $n$ and depth $L$ can compute $f_p$, namely $p$ can be at most $L(n-1)$. Combining these two facts together, we obtain that for any $n$ and $L$, by setting $p = (L+1)(n-1)$, $f_p$ can be computed with $L+1$ layers of $n$ neurons but not with $L$ layers.

We observe that while the increased memory is the key element of our proof, there may exist other forms of expressivity gains coming from depth in linear RNNs.  
\begin{figure}[h]
\centering
\includegraphics[width=0.43\textwidth]{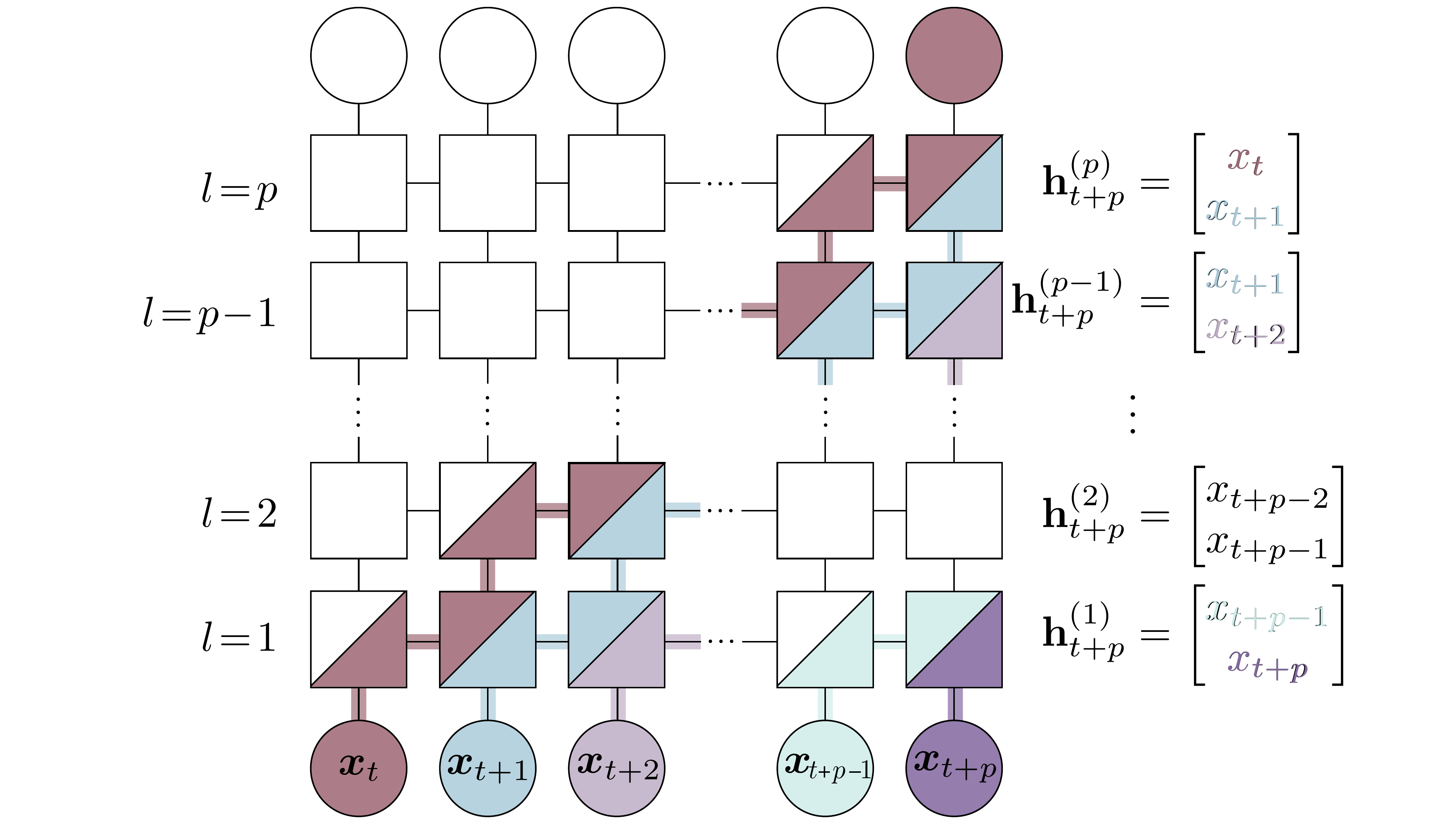}
\caption{Information flow of RNN's computation of $f_p$ for $n=2$ and $L=p$.}\label{fig:sketch_constrution}
\vspace{-0.4cm}
\end{figure}

Since the proof of strict inclusion for Theorem~\ref{thm:deepRNN} relies on increasing memory capacity via the additional hidden units of the extra layer, a natural next step is to compare deep and shallow RNNs with latent representations of the same dimension, i.e., the same total number of hidden units. Concretely, we compare a deep network of hidden dimension $n$ and depth $L$ to a shallow network with hidden dimension $nL$. The following proposition offers a perspective that differs from Theorem~\ref{thm:deepRNN}.
\begin{proposition}[informal\footnote{See Appendix~\ref{apx:prop(n,l)vs(nl,1)} for formal statement and its proof.}]\label{thm:(n,l)vs(nl,1)}
    For $n>1$ and $L>1$, a single-layer linear RNN of hidden size $nL$ has a greater latent representational capacity than a linear RNN of size $n$ and depth $L$. 
\end{proposition}
This proposition indicates that, provided the same hidden capacity~(i.e., number of hidden units), a shallow linear RNN is strictly more expressive than a deeper one.  
This result provides insight into the inner mechanics of RNNs. The latent representation of the shallow network is a vector in $\Rbb^{nL}$, whereas that of the deep network is a concatenation of $L$ vectors in $\Rbb^{n}$. This block structure imposed by depth constrains the expressive power of the latent representation, giving the shallow network a larger representational capacity, which is in direct contrast with Theorem~\ref{thm:deepRNN}.

However, this gain in capacity comes at a higher parameter cost. While the number of hidden units grows linearly with both hidden size and depth, the number of parameters in an RNN increases linearly with $L$ but quadratically with $n$.
The last theorem of this section shows that, if we consider parameter counts instead of number of hidden units, increasing depth can lead to a strict gain in expressiveness.
More precisely, the theorem shows that, for any depth, as soon as the hidden size is large enough, a linear RNN can compute functions that cannot be computed by shallower RNN with~(at most) the same number of parameters. 

\begin{theorem}\label{thm:params}
Considering linear RNNs with input dimension $d=1$, let $\params(n,L)$ denote the number of parameters of a $L$-layers RNN with hidden size $n$ \footnote{One can check that $\params(n,L)=(2L-1)n^2+(L+1)n$~(excluding initial hidden states as parameters count, though including them would not change the result).}.

For any depth $L$ and any hidden size $n\geq 4$, there exists a function $f \in \hRNN{n,L}$ such that, for all $\tilde{L} < L$ and $\tilde{n}$, if $f \in \hRNN{\tilde{n}, \tilde{L}}$ then $\params(\tilde{n}, \tilde{L})>\params(n,L)$.
\end{theorem}
\vspace{-0.4cm}
\paragraph{Sketch of proof}
To prove this theorem, we show that the difference in number of parameters $\params(\tilde{n}, \tilde{L})-\params(n,L)$ required to compute the copy function $f_p$ introduced in Equation~\ref{eq:fp} with $p=(L+1)(n-1)$ evolves as an upward parabola with respect to $n$, and thus will be positive as soon as $n\geq 4$. The complete proof is in the appendix. 

Overall, Theorem~\ref{thm:params} states that there are functions, in particular those requiring memory, for which increasing the number of layer is the parameter-efficient approach to increase the capacity of the network. However, the optimal balance between number of layers and hidden size depends on the task at hand. 
It is worth emphasizing that the gain in efficiency brought by depth is, in essence, independent of the activation of the network. Therefore, Theorem~\ref{thm:params} can be relevant to other recurrent architectures. 
Figure~\ref{fig:sec3.1} summarizes the theoretical findings on linear RNNs presented in this section. 

\begin{figure}[h]
\centering
\includegraphics[width=0.47\textwidth]{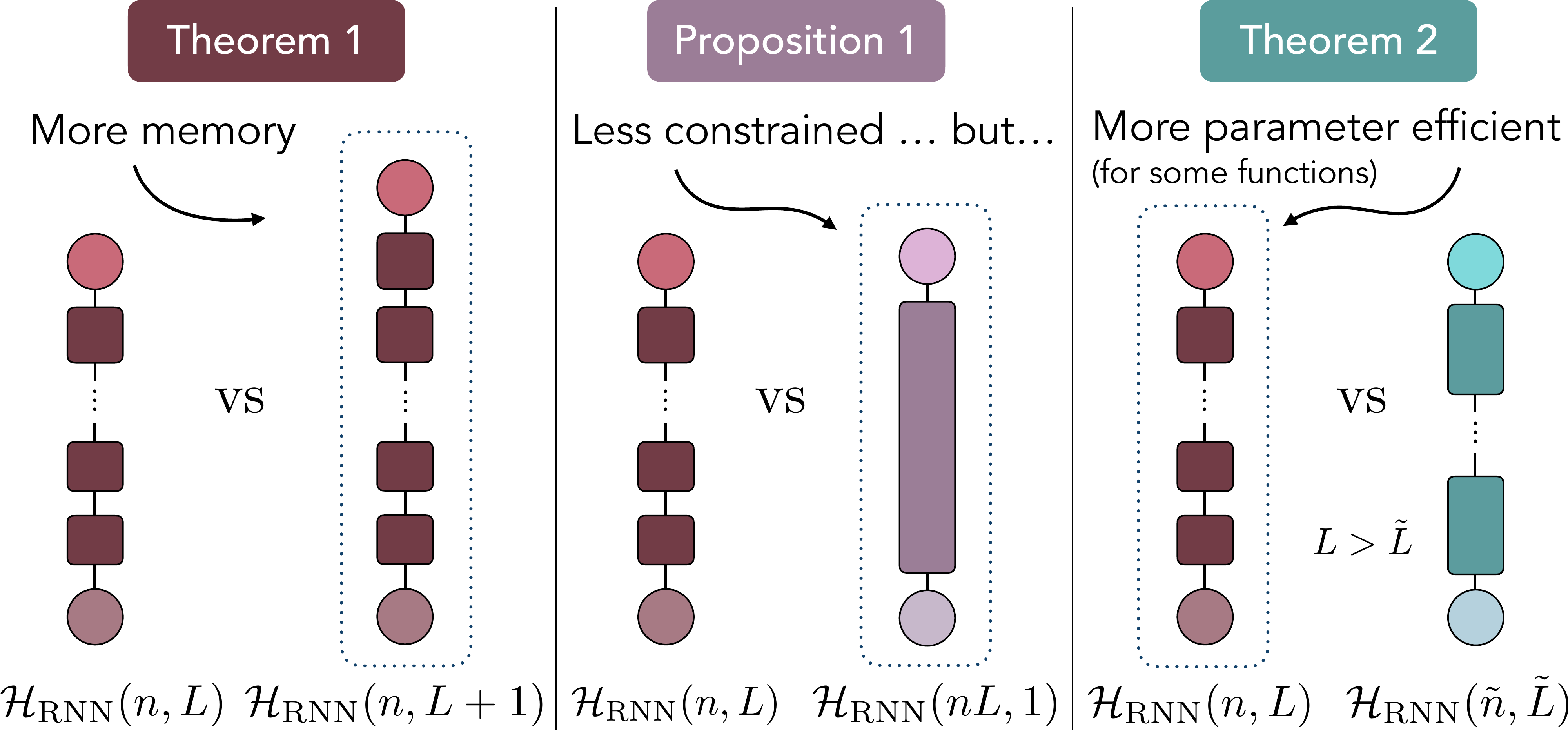}
\caption{Summary of theoretical results on linear RNNs: Adding layers increases memory capacity (Thm.~\ref{thm:deepRNN}), a single layer (for fixed number of units) offers more flexibility (Prop.~\ref{thm:(n,l)vs(nl,1)}), but raising depth rather than width~($n$) can be more parameter-efficient (Thm.~\ref{thm:params}).}\label{fig:sec3.1}
\vspace{-0.4cm}
\end{figure}

\subsection{Role of Depth in BIRNNs: Increasing Complexity via Higher-Order Interactions}\label{sec:2rnn}

We now turn to analyzing the benefits of depth in second-order RNNs.
As in the previous section, we consider linear architectures to abstract away the effect of composing nonlinear activation functions. Additionally, in order to isolate the effect of multiplicative interactions, our formal results focus on linear BIRNNs, rather than 2RNNs.
We begin by observing that even in the absence of nonlinear activations, shallow BIRNNs compute polynomial functions of their inputs~(See App.~\ref{apx:proplinearBIRNN} for formal statement). In contrast with linear RNNs, which compute linear transformations regardless of the network depth, increasing the number of layers in BIRNNs broadens the class of functions they can represent. Specifically, one can show that the maximum degree of the resulting polynomials grows exponentially with depth. Thus, by increasing the number of layers in a BIRNN we can compute higher-order polynomials. This is the key observation to prove the following theorem showing that, similarly to RNNs, there is a strict gain in expressivity with depth in linear BIRNNs. 
\begin{theorem}\label{thm:BIRNN} For $n>1$ and $L\geq1$, $\hBIRNN{n,L} \subsetneq \hBIRNN{n,L+1}$ for linear BIRNNs.
\end{theorem}

\paragraph{Sketch of proof}
To prove inclusion, we show that a BIRNN with $L+1$  layers can produce in its $L^{\text{th}}$  layer the same latent vectors $\hb_t^{(L)}$’s as a BIRNN with $L$ layers, and transfer these hidden states without further modification to the last layer. 
The proof of the strict inclusion relies on showing that in BIRNNs of depth $L$, the polynomial dependency of the second hidden vectors in the first inputs is of degree at most $L$. 
We then consider a specific case that achieves the upper bound for a network of depth $L+1$, and conclude that this function cannot be computed with fewer layers. See Appendix~\ref{apx:thmBIRNN} for the complete proof. 

Note that increasing the hidden size of BIRNNs does not affect the maximal degree of the input’s polynomial representation the way depth does. Consequently, the expressive power brought by depth in these models is always increasing, regardless of hidden size. This contrasts with Theorem~\ref{thm:deepRNN} for linear RNNs, where the gain in memory capacity arises from the number of hidden units, allowing hidden size to compensate for depth. 
Since 2RNNs are RNNs augmented by a bilinear term, i.e. a BIRNN, they benefit from both the increased memory capacity provided by extra hidden units~(Theorem~\ref{thm:deepRNN}) and the higher-order interactions coming from multiplicative terms~~(Theorem~\ref{thm:BIRNN}). One might say, we are far from the shallow now.

It is worth noting that, while 2RNNs are not commonly used in practice, they provide a formal framework generalizing multiplicative interactions in recurrent architectures. As a result, the insights from  Theorem~\ref{thm:BIRNN} can inform our understanding of models incorporating multiplicative mechanisms, such as gated RNNs (e.g. LSTMs~\citep{hochreiter1997long}, GRUs~\cite{cho2014properties}), time-variant SSMs (e.g. Mamba~\citep{gu2023mamba}), Multiplicative Integration RNNs (MIRNNs~\citep{wu2016multiplicative,levine2018benefits}), and CP(BI)RNNs~\citep{lizaire2024tensor,sutskever2011generating}. In particular, for CPBIRNNs we can infer from Thm.~\ref{thm:BIRNN} that increasing depth leads to a strict inclusion, as formalized in the following corollary~(Proof in App.~\ref{apx:cor-cprnn}).

\begin{corollary}\label{cor:cprnn}
    For $n>1$, $L\geq1$ and any $\tilde{R}\geq R >1$, $\hCPBIRNN{n,L,R} \subsetneq \hCPBIRNN{n,L+1, \tilde{R}}$ for linear CPBIRNNs.
\end{corollary}

Corollary~\ref{cor:cprnn} introduces the additional hyperparameter influencing CP(BI)RNNs expressivity: $R$, the rank of the CP decomposition parameterizing their second-order term. This corollary states that, as long as the rank is maintained at least at the same level, adding layers strictly increases the expressive capacity of the network. 
For single-layer CP(BI)RNNs, it has been shown that the rank is an effective tuning parameter for strictly increasing expressivity, up to a saturation point $\Rmax$, the maximal CP rank for a family of tensors sharing the same dimensions, $R^{d_1, d_2 , d_3}_{\max} = \max \{\cprank{\T} \mid \T \in \R^{d_1\times d_2 \times d_3}\}$.
The following theorem verifies that this effect of the rank is not altered by depth.
\begin{theorem}\label{thm:cprnn}
    For $n>1$, $L\geq1$ and any $R\leq R_\textrm{max}$, $\hCPBIRNN{n,L,R} \subsetneq \hCPBIRNN{n,L,R+1}$ for linear CPBIRNNs if $R<n$.
\end{theorem}

\paragraph{Sketch of proof} To prove inclusions it suffices to find explicit parameterizations such that there is equality between the latent vectors $\vec h^t$'s, as in previous sketches of proof. 
To show strict inclusion, that is $\hCPBIRNN{n,L,R+1} \not \subset \hCPBIRNN{n, L, R}$, the idea is to consider the dimension of the image of $\hb_t$. Independently of depth, it is bounded by $R$ for CPBIRNN of rank $R$ while we can find a CPBIRNN of rank $R+1$ whose hidden vector has an image of dimension $R+1$. The complete proof can be found in Appendix~\ref{apx:thm-cprnn}.

Theorem~\ref{thm:cprnn} generalizes the known effect of the rank in single-layer CP(BI)RNNs to deeper architectures, showing that increasing depth does not diminish the expressive benefits provided by a higher rank. In other words, the rank remains an effective lever for controlling expressivity even as the network becomes deeper, allowing to tune both depth and rank independently to achieve the desired expressive power.

\begin{figure}
\centering
\includegraphics[width=0.45\textwidth]{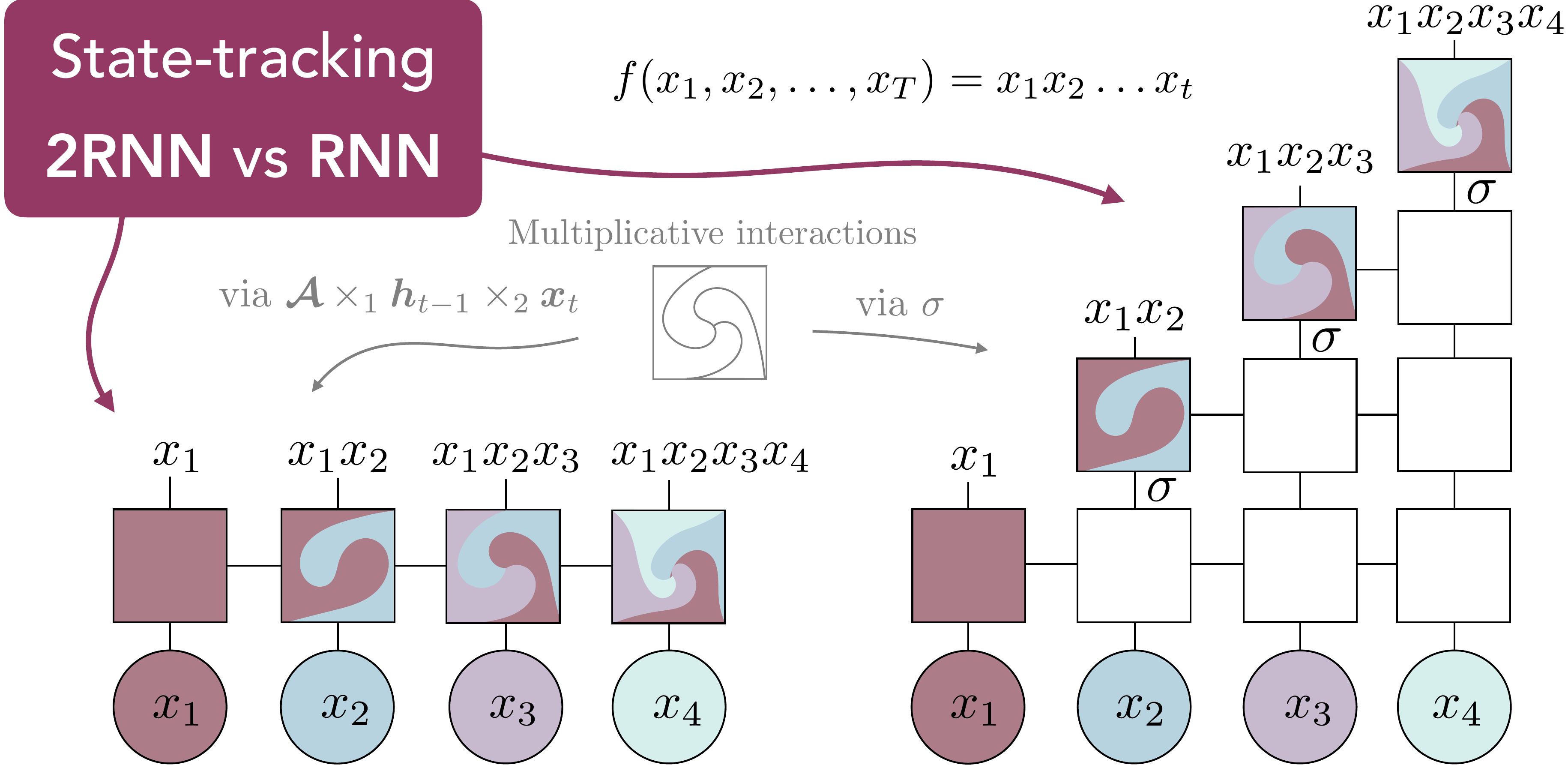}
\caption{State-tracking information flow. In 2RNNs, multiplicative interactions are performed by bilinear products, while in RNNs they need nonlinear activations $\sigma$, thus moving up a layer if applied only in depth.}\label{fig:state-tracking}
\vspace{-0.4cm}
\end{figure}
\subsection{Depth and Multiplicative Interactions: Two Distinct Forms of Expressivity}\label{sec:theoryWFA}
The previous section showed that multiplicative interactions with depth impact expressivity in a manner similar to the composition of nonlinear activations. Both expand the class of functions a model can compute by enabling more complex computations. This naturally raises the question of whether the effects of these two architectural designs on expressive power differ in some way.
An avenue to address this question is to consider the connections 2RNNs have with automata theory and formal languages. In particular, linear BIRNNs are equivalent to weighted finite automata~(WFA)~\citep{li2022connecting} whose computation involves \emph{state-tracking}, i.e., maintaining a vector state that evolves multiplicatively with each input to summarize and aggregate information over the sequence. 
Interestingly, \citet{merrill2024illusion} showed that single-layer state space models, which are linear recurrent networks, cannot perform state-tracking. A natural question to ask is if depth changes this limitation. 
Here, we investigate whether the composition of nonlinear activation functions along depth (i.e. across layers but not time, as in Equation~\ref{eq:RNNonlydepth}), enables an RNN to perform the same computations as a 2RNN. In particular, whether it could state-track. The following theorem provides a negative answer to this question.

\begin{theorem}\label{thm:WFA}
There exists a function computed by a single-layer 2RNN that cannot be computed by any RNN of arbitrary (constant, finite) depth and width with nonlinear activation applied only in depth. 
\end{theorem}
\begin{figure*}
\centering
\includegraphics[width=0.85\textwidth]{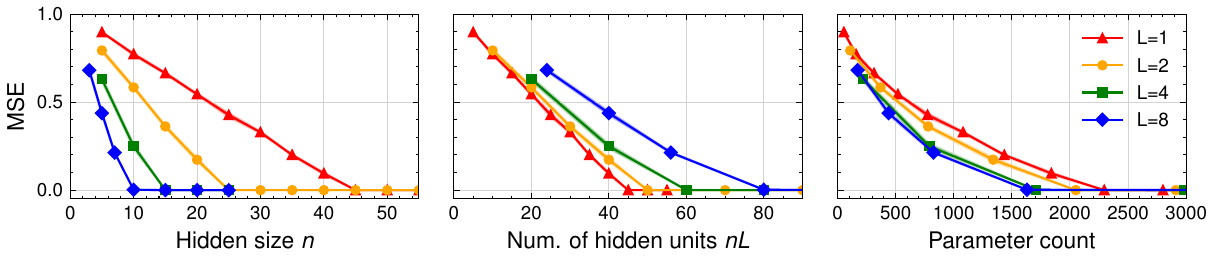}
\caption{Mean squared error~(MSE) on test set for the copy task~(lag 8) of linear RNNs with respect to hidden size~(left), number of hidden units~(center) and number of parameters~(right) for varying depth.}\label{fig:copy_real}
\vspace{-0.1cm}
\end{figure*}
\begin{figure*}[b]
\centering
\includegraphics[width=0.85\textwidth]{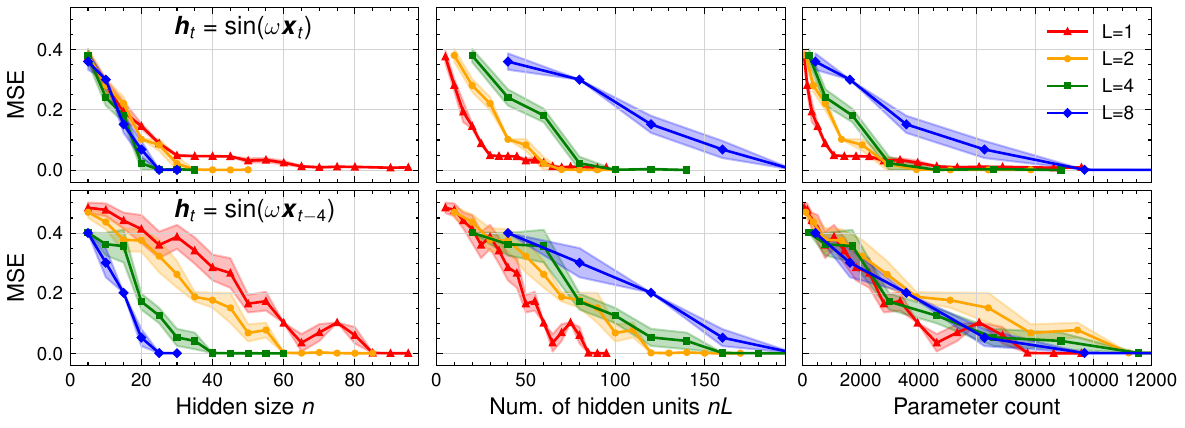}
\caption{Test mean squared error~(MSE) for sinus~(top) and copy-sinus~(lag 4, bottom) of $tanh$ activated RNNs in depth as a function of hidden size~(left), number of units~(center) and  parameters~(right) for varying depth.}\label{fig:copy_sin}
\end{figure*}
\paragraph{Sketch of proof} The demonstration of this theorem relies on the observation that, in (first-order) RNNs, any multiplicative interaction must occur through the application of a nonlinear activation function, i.e. it requires moving up a layer. Since 2RNNs perform one such multiplicative interaction at each time step, an RNN would need a depth (or hidden size) proportional to the sequence length to replicate this behavior as illustrated in Figure~\ref{fig:state-tracking}. Consequently, no deep RNN with nonlinear activations applied only in depth can perform state-tracking on arbitrarily long sequences. The complete proof can be found in Appendix~\ref{apx:thm-wfa}.

The functions implied in Theorem~\ref{thm:WFA} are state-tracking ones. Consequently, this result shows that deep RNNs applying nonlinear activations only in depth cannot, in general, perform state-tracking tasks, regardless of the number of layers. Importantly, this holds for arbitrarily complex nonlinearities. This result reveals a fundamental separation in the expressivity conferred by multiplicative interactions versus nonlinear activations, highlighting that these two architectural designs impose different inductive biases. 

\section{Experiments}
We assess the practical implications of our theoretical results through experiments on synthetic and real data using RNNs, CPRNNs, and S4 models of varying depth. CPRNNs are employed, rather than BIRNNs or 2RNNs, to avoid handling the third-order tensor growth.
First, to validate the theory from Section~\ref{sec:theoryRNNs}, we test linear RNNs on the memorization task $f_p$ (Eq.~\ref{eq:fp}), then we evaluate the nonlinear case using a sinusoidal transformation.
Next, we explore the insights from Section~\ref{sec:theoryWFA} with the parity task, which emphasizes multiplicative state-tracking over memorization, and compare the effects of nonlinearities in recurrence versus depth.

We then evaluate RNNs, CPRNNs, and S4 on language modeling using the tiny Shakespeare dataset~\citep{Karpathy2015}.
Lastly, we analyze the impact of depth in S4 models across multiple tasks using the Long Range Arena benchmark~\citep{tay2021long}. All models were trained with Adam optimizer~\citep{kingma2015adam}  and early stopping. Plots show averages over 3–5 random seeds, with shaded areas indicating standard deviations. S4 weights follow~\citep{gu2021efficiently}; others are initialized with $\mathcal{U}[-\frac{1}{\sqrt{n}},\frac{1}{\sqrt{n}}]$. Details specific to each experiments are given in the appendix.
\begin{figure*}
\centering
\includegraphics[width=0.9\textwidth]{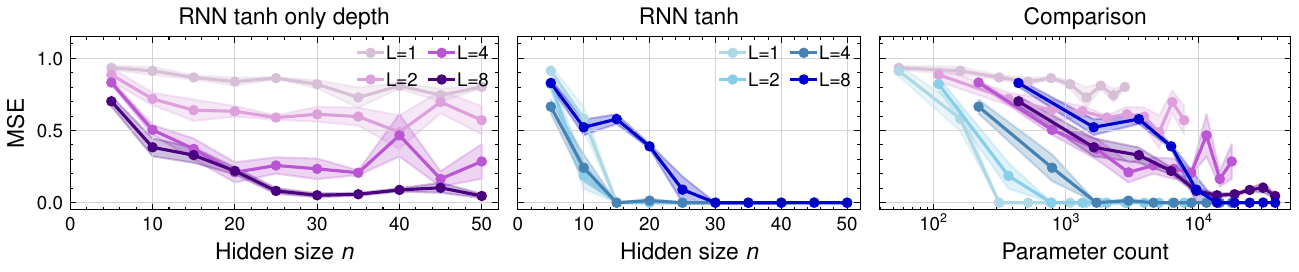}
\caption{Test MSE for the parity task versus hidden size $n$ on RNNs with $tanh$ activations in depth~(left) or recurrence~(center) for varying depth. Right panel compares activations by number of parameters.}\label{fig:parity}
\vspace{-0.3cm}
\end{figure*}
\subsection{Copy, sinus, \& sinus-copy tasks}\label{sec:copy}
Our first synthetic experiment is the copy task $f_p$ from Eq.~\ref{eq:fp}, modified to output all input dimensions. We use a lag $p=8$ and inputs in $\mathbb{R}^5$ sampled from $\mathcal{N}(0,1)$. The training set contains 10,000 sequences of length 16; validation and test sets each contain 2000 sequences.

As predicted by Theorem~\ref{thm:deepRNN}, Figure~\ref{fig:copy_real} shows that deeper models require a smaller hidden size to achieve zero loss. Indeed, depth strictly increases expressivity and, empirically, leads to better performances. However, the hidden capacity is also augmented by depth. When plotting the mean square error~(MSE) against the number of hidden units~(center), the trend reverses: shallower models outperform deeper ones for a fixed unit budget, corroborating Proposition~\ref{thm:(n,l)vs(nl,1)}. Finally, Theorem~\ref{thm:params} is validated empirically for this task by the right panel, where the models performance with respect to the number of parameters gradually increases with depth.

We then study the extent to which these trends hold in a nonlinear setting, using a sinusoidal transformation of the copy task~(Figure~\ref{fig:copy_sin}). With $p=0$, the task becomes $\hb_t = \sin(\omega \xb_t)$, which requires no memory and serves as a baseline for the copy-sinus case with a lag of 4, $\hb_t = \sin(\omega \xb_{t-4})$. The models have $\tanh$ activation in depth, keeping the recurrence linear.

For $p=0$~(top), depth reduces the hidden size needed to learn the task, with a marked gap between $L=1$ and $2$, suggesting that the sinus transformation is hard to learn by a one-layer model. Nonetheless, $L=2$ is the first to achieve a zero-loss with respect to the number of parameters, which empirically shows that depth does not provide a parameter efficiency benefit for this specific task.

The sinus-copy case (bottom) combines memory and nonlinearity, both imposing thresholds on the minimum hidden size required for convergence towards a solution. The former gives an advantage in terms of the number of parameters while the latter does not. As shown in the bottom-left panel, all models perform worse than in the memoryless case, but deeper ones ($L=4$, $8$) degrade less than shallower ones ($L=1$, $2$). Consequently, for an equal parameters budget~(right panel) the deeper RNNs become competitive, $L=8$ achieving zero loss with the least amount of parameter.

\subsection{Parity task}\label{sec:parity}
We now examine parity, a state-tracking task that does not involve memory, to corroborate the theory from Section~\ref{sec:theoryWFA}.
At each time step, the model must identify whether the number of $-1$s seen so far in a sequence of 1s and $-1$s is even or odd. We evaluate the MSE on sequences of length 20 with 5-dimensional inputs, where the task is performed independently in each dimension. Theorem~\ref{thm:WFA} states that this problem cannot be compensated by depth-wise activations. It is instead biased towards second-order interactions between input and hidden state, which is confirmed empirically by linear 2RNNs solving it with a hidden size as low as 5.

We compare RNNs with $\tanh$ activations applied recurrently (as in Definition~\ref{def:RNN}) or only in depth~(as in Equation~\ref{eq:RNNonlydepth}). As shown in Figure~\ref{fig:parity}~(left), shallow models with only depth-wise activations~($L=1$ and $2$) fail to learn the task, while deeper ones~($L=4$ and $8$) approach zero loss for hidden sizes over 50. In contrast, RNNs with recurrent activations learn the task with fewer than 20 hidden units, even with only one layer~(center panel), validating empirically Theorem~\ref{thm:WFA}.

Comparing the two activation schemes~(right panel), we observe distinct trends: recurrently activated RNNs gain parameter efficiency with reduced depth, while depth-only models show no clear depth dependency. This suggests models are not copying input's information, consistent with the task not requiring memory. It also indicates that parity is not among the functions behind the strict inclusion in Theorem~\ref{thm:params}.

\subsection{Language modeling on the tiny Shakespeare dataset}\label{sec:shakespeare}
\begin{figure*}[h]
\centering
\includegraphics[width=0.98\textwidth]{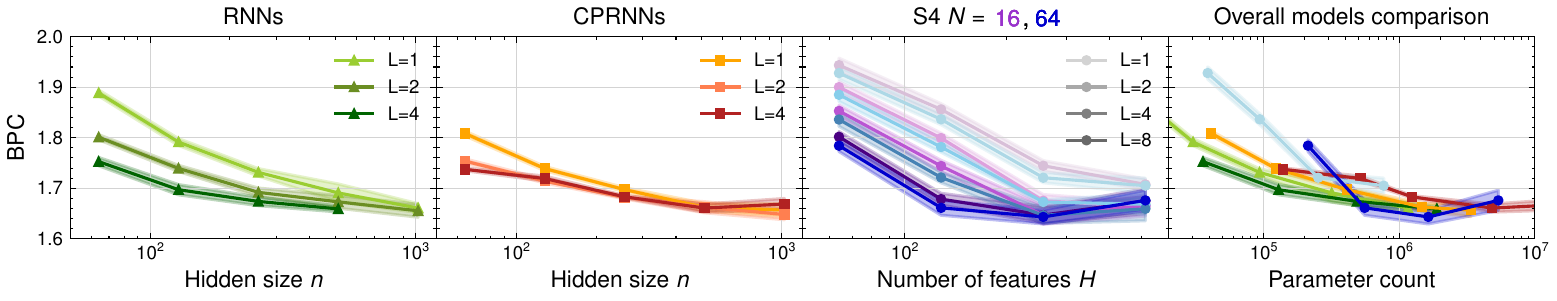}
\caption{Test bits-per-character~(BPC) of RNNs~(first), CPRNNs~(second) versus hidden size $n$ and S4~(third) versus features $H$ across depths on tiny Shakespeare, fourth panel compares models by parameter count.}\label{fig:shakespeare}
\end{figure*}
We study the effect of depth on RNNs, CPRNNs, and S4's ability to model natural language using the tiny Shakespeare dataset~\citep{Karpathy2015}, composed of 40,000 lines from the author's corpus. Models were trained at the character level using negative log likelihood. RNNs and CPRNNs use $\tanh$ activation applied recurrently, and S4 follows the architecture used in~\citep{gu2021efficiently} for language modeling.

Figure~\ref{fig:shakespeare} shows a clear and consistent  benefit to depth in RNNs, seen both in terms of hidden size~(first panel) and parameter count~(fourth), where deep RNNs outperform all others almost everywhere. This resembles the copy task behavior~(Fig.~\ref{fig:copy_real}), suggesting that language modeling either involves memorization or another expressive aspect with similar parameter efficiency.

For CPRNNs~(second panel), the benefit of added layers is much less prominent. There is some improvement from $L=1$ to $L=2$, but it saturates at $L=4$, suggesting that the polynomial class generated by two layers suffices. Note however that the rank was not finely tuned. Also, for $n=1024$ and all $L=4$ models, slower optimization was needed for stability, possibly explaining the lower performance. Still, for $L=1,2$ and $n=64$–$512$, CPRNNs outperform RNNs.

For S4, we use depths up to $L=8$ and state dimensions $N=16$ and $64$. Depth brings consistent gains across feature counts $H$~(third panel), but with fixed parameter budgets~(fourth), more layers offer no clear benefit. Below 1M parameters, S4 underperforms both RNNs and 2RNNs regardless of depth which is reminiscent of the gap between RNNs activated recurrently or only in depth in the parity task~(Figure~\ref{fig:parity}), suggesting a similar bias toward multiplicative state-tracking. Overall, deep S4 with at least 256 features (~1M parameters) perform best, reaching BPCs similar to deep RNNs.

\subsection{Long Range Arena}
To assess the role of depth on a variety of tasks, we conclude our empirical investigation by comparing S4 models on Long Range Arena problems~\citep{tay2021long}.
The same experimental configurations as in~\citep{gu2021efficiently} were adopted. 
Table~\ref{tab:lra-table} shows that the effect of depth is not the same across all tasks. Performance on Images and ListOps are improved by increasing the number of layers, but eventually saturates, whereas Pathfinders exhibits continued improvement with depth. 
Interestingly, the efficiency of the Retrieval task is mostly independent of the number of layers. Within our theoretical framework, this contrast can be understood as a difference in memory requirement. Indeed, Retrieval tests a model's ability to store compressed information which is less memory intensive in comparison to Pathfinders which requires remembering all paths in an image. Overall, these results support the task-dependency paradigm of depth's benefit in recurrent architectures inferred by Theorems~\ref{thm:(n,l)vs(nl,1)} and~\ref{thm:params}.

\begin{table}[h]
    \centering
    \begin{tabular}{ccccc}
        \toprule
        L & Retrieval & Images & ListOps & Pathfinder \\
        \midrule
        2 & 0.9139 & 0.8581 & 0.5806 & 0.8534 \\
        3 & \textbf{0.9170} & 0.8765 & 0.6154 & 0.9328 \\
        6 & 0.9099 & \textbf{0.8980} & \textbf{0.6159} & 0.9619 \\
        8 & 0.9118 & 0.8913 & 0.6154 & \textbf{0.9673} \\
        \bottomrule
    \end{tabular}
    \caption{Test accuracies on Long Range Arena of S4 for varied depth with fixed parameter budget.} 
    \label{tab:lra-table}
\end{table}

\section{Conclusion}
We studied how depth affects the expressivity of recurrent networks by isolating it from nonlinear activations, using linear RNNs to analyze depth–recurrence interactions and linear BIRNNs for depth with multiplicative interactions. We formally show that adding layers increases memory capacity and allows higher-order interactions, and that some single-layer 2RNN functions, like state-tracking, cannot be realized by deep RNNs with nonlinearities only in depth. Experiments on synthetic and real tasks confirm that depth generally improves performance, with gains and parameter efficiency depending on task.
This work motivates further explorations of how architectural choices, such as gating mechanisms or time-variant versus time-invariant structures in SSMs, affect expressivity. Another important direction is to deepen our understanding of how these designs interact with optimization. 
Such insights could clarify how robust the gains in expressive power are in practice and provide principled guidance to design models that are both theoretically expressive and practically effective.
\section*{Acknowledgement}
M. Lizaire’s research is supported by NSERC (Vanier Scholarship) and IVADO (PhD Excellence Scholarship); G. Rabusseau’s by NSERC and the CIFAR AI Chair program. We also acknowledge NVIDIA for providing computational resources. We are grateful to Pascal Jr Tikeng Notsawo for his valuable feedback and support.
\bibliographystyle{unsrtnat}
\bibliography{biblio_depth}


\clearpage
\appendix

\onecolumn
\aistatstitle{On the Role of Depth in the Expressivity of RNNs: \\
Supplementary Materials}

\input{new_supplements}

\end{document}

%% file: macros.tex
\usepackage{mathtools}
\DeclarePairedDelimiter{\ceil}{\lceil}{\rceil}

\newcommand{\params}{\#\mathrm{params}}

\newcommand{\Rmax}{R_{\mathrm{max}}}
\newcommand{\diag}{\mathrm{diag}}

 \usepackage{relsize}

\usepackage{etoolbox}
\usepackage{xparse}

\renewcommand{\enspace}{}
\newcommand{\R}{\Rbb}

\renewcommand{\vec}[1]{\ensuremath{\mathbf{#1}}}
\newcommand{\vecs}[1]{\ensuremath{\mathbf{\boldsymbol{#1}}}}
\newcommand{\mat}[1]{\ensuremath{\mathbf{#1}}}
\newcommand{\mats}[1]{\ensuremath{\mathbf{\boldsymbol{#1}}}}
\newcommand{\ten}[1]{\mat{\ensuremath{\boldsymbol{\mathcal{#1}}}}}

\usepackage{pgffor}

\foreach \x in {A,...,Z}{%
\expandafter\xdef\csname \x bb\endcsname{\noexpand\ensuremath{\noexpand\mathbb{\x}}}
}

\foreach \x in {A,...,Z}{%
\expandafter\xdef\csname \x cal\endcsname{\noexpand\ensuremath{\noexpand\mathcal{\x}}}
}

\foreach \x in {A,...,Z}{%
\expandafter\xdef\csname \x t\endcsname{\noexpand\ensuremath{\noexpand\ten{\x}}}
}

\foreach \x in {A,...,Z}{%
\expandafter\xdef\csname \x b\endcsname{\noexpand\ensuremath{\noexpand\mat{\x}}}
}

\foreach \x in {a,...,r}{%
\expandafter\xdef\csname \x b\endcsname{\noexpand\ensuremath{\noexpand\vec{\x}}}
}
\foreach \x in {t,...,z}{%
\expandafter\xdef\csname \x b\endcsname{\noexpand\ensuremath{\noexpand\vec{\x}}}
}

\newcommand{\A}{\mat{A}}
\newcommand{\B}{\mat{B}}
\newcommand{\C}{\mat{C}}

\newcommand{\T}{\ten{T}}

\newcommand{\M}{\mat{M}}

\newcommand{\V}{\mat{V}}

\renewcommand{\v}{\vec{v}}

\newcommand{\x}{\vec{x}}

\newtheorem*{theorem*}{Theorem}
\newtheorem*{corollary*}{Corollary}%
\newtheorem*{proposition*}{Proposition}%
\ifcsdef{theorem}{}{
\newtheorem{theorem}{Theorem}%
\newtheorem{proposition}{Proposition}%
\newtheorem{lemma}{Lemma}%
\newtheorem{definition}{Definition}%
\newtheorem{corollary}{Corollary}%
}
\newtheorem*{pbm*}{Problem}%
\newtheorem*{algo*}{Algorithm}%

\DeclareMathOperator*{\rank}{rank}

\renewcommand{\Im}{\mathrm{Im}}

\newcommand{\nstates}{n}

\newcommand{\szerosymbol}{\alpha}
\newcommand{\szero}{\vecs{\szerosymbol}}

\newcommand{\sinfsymbol}{\omega}
\newcommand{\sinf}{\vecs{\sinfsymbol}}

\DeclareDocumentCommand{\wa}{  O{A} O{\szero} O{\sinf} }%
{(#2,\{\mat{#1}^\sigma\}_{\sigma\in\Sigma},#3)}
\DeclareDocumentCommand{\waR}{  O{A} O{\Rbb^\nstates} O{\szero} O{\sinf} }%
{(#2,#3,\{\mat{#1}^\sigma\}_{\sigma\in\Sigma},#4)}

\newcommand{\vvsinfsymbol}{\Omega}
\newcommand{\vvsinf}{\mats{\vvsinfsymbol}}
\DeclareDocumentCommand{\vvwa}{  O{A} O{\szero} O{\vvsinf} }%
{(#2,\{\mat{#1}^\sigma\}_{\sigma\in\Sigma},#3)}

\newcommand{\tzerosymbol}{\alpha}
\newcommand{\tzero}{\vecs{\tzerosymbol}}

\newcommand{\tinfsymbol}{\omega}
\newcommand{\tinf}{\vecs{\tinfsymbol}}

\DeclareDocumentCommand{\wta}{ O{T} O{\Rbb^\nstates} O{\tzero} O{\tinf} O{\Fcal}}%
{(#2,#3,\{\ten{#1}^g\}_{g\in #5_{\geq 1}},\{#4^\sigma\}_{\sigma\in #5_0})}

\DeclareDocumentCommand{\trees}{g}{\IfNoValueTF{#1}{\mathfrak{T}}{\mathfrak{T}_{#1}}}
\DeclareDocumentCommand{\contexts}{g}{\IfNoValueTF{#1}{\mathfrak{C}}{\mathfrak{C}_{#1}}}

\newcommand{\gwmprod}{\diamond}

\DeclareDocumentCommand{\gwm}{  O{M} O{\Fbb^\nstates}}{(#2, \{\ten{#1}^x\}_{x\in\Sigma})}
\DeclareDocumentCommand{\gwmcirc}{  O{M} O{\Rbb^\nstates}}{(#2, \{\mat{#1}^\sigma\}_{\sigma\in\Sigma})}
\DeclareDocumentCommand{\dgwm}{ O{M} O{\Fbb^\nstates}}{(#2, \{\ten{#1}^x\}_{x\in\Sigma},\gwmprod)}

\newcommand{\cprank}[1]{\mathrm{rank}_{\mathrm{CP}}(#1)}

\newcommand{\hRNN}[1]{\mathcal{H}_{\mathrm{RNN}}(#1)}
\newcommand{\hRNNconcat}[1]{\mathcal{H}_{\mathrm{RNN}}^{(\|)}(#1)}
\newcommand{\hSORNN}[1]{\mathcal{H}_{2\mathrm{RNN}}(#1)}
\newcommand{\hBIRNN}[1]{\mathcal{H}_{\mathrm{BIRNN}}(#1)}
\newcommand{\hCPRNN}[1]{\mathcal{H}_{\mathrm{CPRNN}}(#1)}

\newcommand{\hCPBIRNN}[1]{\mathcal{H}_{\mathrm{CPBIRNN}}(#1)}
\newcommand{\linearmaps}[1]{\mathcal{L}^{#1}}
\newcommand{\CP}[1]{ [\![#1]\!]}

\newcommand{\fr}[1]{f^{\to}_{#1}}
\newcommand{\fu}[1]{f^{\uparrow}_{#1}}

\usepackage{pgffor}
\foreach \x in {A,...,Z}{%
\expandafter\xdef\csname \x bb\endcsname{\noexpand\ensuremath{\noexpand\mathbb{\x}}}
}

\foreach \x in {A,...,Z}{%
\expandafter\xdef\csname \x cal\endcsname{\noexpand\ensuremath{\noexpand\mathcal{\x}}}
}

\usepackage{pgffor}
\foreach \x in {A,...,Z}{%
\expandafter\xdef\csname \x ten\endcsname{\noexpand\ensuremath{\noexpand\ten{\x}}}
}

\foreach \x in {A,...,Z}{%
\expandafter\xdef\csname \x mat\endcsname{\noexpand\ensuremath{\noexpand\mat{\x}}}
}

\foreach \x in {a,...,z}{%
\expandafter\xdef\csname \x vec\endcsname{\noexpand\ensuremath{\noexpand\mat{\x}}}
}

\usepackage[symbol]{footmisc}

%% file: new_supplements.tex
\section{Proofs of theoretical results}
\subsection{Preliminaries for Proofs}
Before presenting the proofs of our theoretical results, we introduce some additional notation, a lemma and its corollary that will be used subsequently as well as a formal definition for CP(BI)RNNs.
\subsubsection{Notation}
\begin{itemize}
    \item $\linearmaps{p,q}$ denotes the space of linear maps from $\Rbb^{p}$ to $\Rbb^{q}$. 
    \item $\delta_{i,j}$ denotes the Kronecker delta, that is $\delta_{i,j} = 1$ if $i=j$ and $0$ otherwise.
    \item $[n, m]$ for $n,m \in \Nbb$ with $m>n$ denotes the interval of integers from $n$ to $m$.
\end{itemize} 

\subsubsection{Expressing $\hb_t^{(l)}$ in terms of $\xb_t$}
We now introduce Lemma~\ref{lemma:induction-2RNN} which is used in the proof of Theorem~\ref{thm:BIRNN},  and then its corollary used to prove Proposition~\ref{thm:(n,l)vs(nl,1)}.
\begin{lemma}\label{lemma:induction-2RNN}
   The hidden states in a 2RNN at time $t$ and layer $l$ can be expressed as 
\begin{align*}
\hb_t^{(l)} &=    \left(\prod_{i=1}^{l}\bar\Ub_{t-1}^{(i)}\right)\xb_{t}+\sum_{j=1}^{l}\left(\prod_{k=j+1}^{l}\bar\Ub_{t-1}^{(k)}\right)\bar\bb_{t-1}^{(j)}, \\ 
\text{where }  \bar\Ub_t^{(l)} &\equiv (\At^{(l)} \times_1 \hb_t^{(l)})^\top + \Ub^{(l)},  \quad  
\bar\bb_t^{(l)} \equiv  \Vb^{(l)} \hb_t^{(l)} + \bb^{(l)}. \nonumber
\end{align*}
\end{lemma}

\begin{proof}
We first note that the bilinear term in the 2RNN (Def. \ref{def:2RNN}) can be expressed as a matrix-vector product, $(\At^{(l)} \times_1 \hb_{t-1}^{(l)})^\top \hb_t^{(l-1)}$. Inserting this in the complete definition of the 2RNN and reorganizing the terms, we obtain 
\begin{align}
\hb_t^{(l)} 
&= [(\At^{(l)} \times_1 \hb_{t-1}^{(l)})^\top +\Ub^{(l)}] \hb_t^{(l-1)} + [\Vb^{(l)} \hb_{t-1}^{(l)} + \bb^{(l)}] \nonumber \\ 
&\equiv \bar\Ub_{t-1}^{(l)} \hb_t^{(l-1)} + \bar\bb_{t-1}^{(l)} \label{eq:hbt-2RNN}
\end{align}

\begin{itemize}[left=10pt]
\item For $l=1$:
From the expression above we have $\hb_t^{(1)} = \bar{\Ub}_{t-1}^{(1)} \hb_t^{(0)} + \bar{\bb}_{t-1}^{(1)}$, which indeed  corresponds to the evaluation of Lemma~\ref{lemma:induction-2RNN} for $l=1$. 
\item For $l$, we assume that this claim holds for $l-1$
 \[
 \hb_{t}^{(l-1)}	= \left(\prod_{i=1}^{l-1}\bar\Ub_{t-1}^{(i)}\right)\xb_{t}+\sum_{j=1}^{l-1}\left(\prod_{k=j+1}^{l-1}\bar\Ub_{t-1}^{(k)}\right)\bar\bb_{t-1}^{(j)} 
 \]
and insert this result in Eq.~\ref{eq:hbt-2RNN} 
\begin{align*}
\hb^{(l)}_t 
&= \bar\Ub_{t-1}^{(l)} \left[  \left(\prod_{i=1}^{l-1}\bar\Ub_{t-1}^{(i)}\right)\xb_{t}+\sum_{j=1}^{l-1}\left(\prod_{k=j+1}^{l-1}\bar\Ub_{t-1}^{(k)}\right)\bar\bb_{t-1}^{(j)}\right] + \bar\bb_{t-1}^{(l)} \\
&=  \left(\prod_{i=1}^{l}\bar\Ub_{t-1}^{(i)}\right)\xb_{t}+\sum_{j=1}^{l-1}\left(\prod_{k=j+1}^{l}\bar\Ub_{t-1}^{(k)}\right)\bar\bb_{t-1}^{(j)} + \bar\bb_{t-1}^{(l)} \\
&=  \left(\prod_{i=1}^{l}\bar\Ub_{t-1}^{(i)}\right)\xb_{t}+\sum_{j=1}^{l}\left(\prod_{k=j+1}^{l}\bar\Ub_{t-1}^{(k)}\right)\bar\bb_{t-1}^{(j)} 
\end{align*}
        
\end{itemize}
This corresponds to Lemma \ref{lemma:induction-2RNN}, which completes the induction and proves the result is true for any~$l$.
\end{proof}
The following corollary to Lemma~\ref{lemma:induction-2RNN} is obtained by setting $\At^{(l)} = 0$, thus considering a RNN instead of a 2RNN.  In this case,  $\bar{\Ub}_{t-1}^{(i)}=\Ub^{(i)}$, and by expanding $\bar{\bb}_t^{(l)}$, we obtain the following result. 
\begin{corollary}\label{cor:induction-RNN}
   The hidden states in a RNN at time $t$ and layer $l$ can be expressed as 
\[
\hb_t^{(l)} = \left(\prod_{i=1}^{l}\Ub^{(i)}\right)\xb_{t}+\sum_{j=1}^{l}\left(\prod_{k=j+1}^{l}\Ub^{(k)}\right)\Vb^{(j)}\hb_{t-1}^{(j)}+\sum_{j=1}^{l}\left(\prod_{k=j+1}^{l}\Ub^{(k)}\right)\bb^{(j)}
\]
\end{corollary}

\subsubsection{Formal definition CP(BI)RNNs}\label{apx:cprnn}
We begin by reviewing the basics of CP decomposition, before introducing a formal definition of CPRNNs, from which we also derive the definition of CPBIRNNs. More details on the single-layer versions of these models can be found in~\cite{lizaire2024tensor}.

First, recall that a CP decomposition of rank $R$ expresses a tensor $\Tt \in \Rbb^{d_1 \times d_2\times d_3}$ as a sum of $R$ rank-one tensors: $\Tt= \sum_{r=1}^{R} \ab_r \circ \bb_r \circ \cb_r \equiv \CP{\Ab,\Bb,\Cb}$ where $\ab_r \in \Rbb^{d_1}$, $\bb_r \in \Rbb^{d_2}$ and $\cb_r \in \Rbb^{d_3}$~(see, e.g., \citep{kolda2009tensor}). The factor matrices $\A \in \R^{d_1\times R}, \B \in \R^{d_2\times R}, \C \in \R^{d_3\times R}$ have the vectors $\ab_r,\bb_r,\cb_r$ as  columns. The CP decomposition reduces the parameter count from $\mathcal O(d^3)$ to $\mathcal O(Rd)$ where $d = max\{d_1, d_2, d_3\}$. We now turn to the definition of (multi-layer) CPRNNs.

\begin{definition}[CPRNN]\label{def:cprnn}
A CP Recurrent Neural Network of depth $L$, hidden size $n$ and rank $R$ is parameterized  by initial hidden state vectors $\hb_0^{(l)} \in \Rbb^n$, weight matrices $\Ab^{(l)},\Bb^{(l)},\Cb^{(l)} \in \Rbb^{n \times R}$ (except for $\Bb^{(1)}\in \Rbb^{d \times R}$), $\Ub^{(l)},\Vb^{(l)} \in \Rbb^{n \times n}$ (except for $\Ub^{(1)}\in \Rbb^{d \times n}$), bias terms $\bb^{(l)}\in \Rbb^n$ and activation functions $\sigma^{(l)}: \Rbb^n \to \Rbb^n$ for $l \in [L]$. Given a sequence of inputs $(\vec{x}_1, \vec{x}_2,..., \vec{x}_T)$, a CPRNN outputs a sequence of hidden states $(\vec{h}_1^{(L)}, \vec{h}_2^{(L)},..., \vec{h}_T^{(L)})$ via the following computation at each time step $t \in [T]$ and layer $l \in [L]$: 
$$\vec{h}_t^{(l)} = \sigma^{(l)}(\CP{\Ab^{(l)}, \Bb^{(l)}, \Cb^{(l)}} \times_1 \vec{h}_{t-1}^{(l)} \times_2 \hb^{(l-1)}_t  + \Vb^{(l)} \vec{h}_{t-1}^{(l)} + \Ub^{(l)} \hb^{(l-1)}_t + \bb^{(l)}).$$
\end{definition}
As in Definition~\ref{def:RNN} and~\ref{def:2RNN}, the activation functions $\sigma^{(l)}$ are applied element-wise. CPBIRNN is obtained from this definition by setting the first-order parameters ($\Vb^{(l)}, \Ub^{(l)}, \bb^{(l)}$) to zero. Moreover, note that the bilinear term of CP(BI)RNNs can be expressed using matrix product:
\begin{align}\label{eq:CP-prod}
\CP{\Ab^{(l)},\Bb^{(l)},\Cb^{(l)}} \times_1 \hb_{t-1}^{(l)} \times_2 \hb_{t}^{(l-1)} 
= \left[ \C^{(l)} \diag({\A^{(l)}}^\top \hb_{t-1}^{(l)}) {\B^{(l)}}^\top \right] \hb_t^{(l-1)}. 
\end{align}
\subsection{Formalization of "linear RNNs compute linear transformations"}\label{apx:proplinearRNN}
The following proposition formalizes the statement \emph{linear RNNs perform linear transformations of their inputs, regardless of the network’s depth}.
\begin{proposition}\label{prop:linearRNNlinear}
Let $h\in \hRNN{n,L}$ be the function computed by a linear RNN. Then, for any $T$, the function $f : \Rbb^{dT} \to \Rbb^{nT}$ mapping~(concatenations of) sequences of inputs $(\xb_1,\cdots,\xb_T)$ to~(concatenations of) sequences of hidden states $(\hb^{(L)}_1,\cdots, \hb^{(L)}_T)$ is linear.
\end{proposition}

\begin{proof}
We show that $\hb^{(L)}_t \in \linearmaps{dT,n}$ for all $t \in [T]$, which implies that their concatenation is in $\linearmaps{dT,nT}$\footnote{Here~(and thereafter), we use the abuse of notation  $\hb^{(L)}_t \in \linearmaps{dT,n}$ to mean that the function that maps $(\xb_1,\cdots, \xb_T)$ to $\hb^{(L)}_t$ is linear.}. The proof consists of a nested induction over $l$ and $t$. 
\begin{itemize}
    \item Base case $l=1$: We show that $\hb^{(1)}_t \in \linearmaps{dT,n}$ for all $t$.
    \begin{itemize}
        \item Base case $t=1$ : $\hb^{(1)}_1 = \Ub^{(1)} \xb_1 + \V^{(1)} \hb^{(1)}_0 +\bb^{(1)} \in \linearmaps{dT,n}$
        \item Inductive step on $t$ (proving for $t$ assuming true for $t-1$):  
        $\hb^{(1)}_t = \Ub^{(1)} \xb_t + \V^{(1)} \hb^{(1)}_{t-1} +\bb^{(1)} \in \linearmaps{dT,n}$ since by induction hypothesis $\hb^{(1)}_{t-1} \in \linearmaps{dT,n}$.
    \end{itemize}
    \item Inductive step on $l$: We show that $\hb^{(l)}_t \in \linearmaps{dT,n}$ assuming $\hb^{(l-1)}_t \in \linearmaps{dT,n}$.
    \begin{itemize}
        \item Base case $t=1$ : $\hb^{(l)}_1 = \Ub^{(l)} \hb^{(l-1)}_1 + \V^{(l)} \hb^{(l)}_{0} +\bb^{(l)} \in \linearmaps{dT,n}$ since $\hb^{(l-1)}_{1} \in \linearmaps{dT,n}$ by induction hypothesis.
        \item Inductive step on $t$ (proving for $t$ assuming true for $t-1$): $\hb^{(l)}_t = \Ub^{(l)} \hb^{(l-1)}_t + \V^{(l)} \hb^{(l)}_{t-1} +\bb^{(l)} \in \linearmaps{dT,n}$ since $\hb^{(l-1)}_{t} \in \linearmaps{dT,n}$ by induction hypothesis on $l$ and $\hb^{(l)}_{t-1} \in \linearmaps{dT,n}$ by induction hypothesis on $t$.
    \end{itemize}
\end{itemize}
\end{proof}

\subsection{Proof of Theorem~\ref{thm:deepRNN}}
\begin{theorem*} For any $n>1$ and $L\geq 1$, $\hRNN{n,L} \subsetneq \hRNN{n,L+1}$ for linear RNNs.
\end{theorem*}

\begin{proof}
~\paragraph{Inclusion: } 
We first show that for any $h \in \hRNN{n,L}$, there exists a function $\tilde{h} \in \hRNN{n, L+1}$ such that $h=\tilde{h}$, that is $\hRNN{n,L} \subseteq \hRNN{n,L+1}$.
Consider the parameters of the RNN computing $h$ : $\{\Ub^{(l)},  \Vb^{(l)}, \bb^{(l)}, \hb_0^{(l)}\}_{l=1}^{L}$. First, we set the parameters of the first $L$ layers of the RNN computing $\tilde{h}$ the same as the one computing $h$ so that the hidden states of the $L^{\text{th}}$ layer are the same for both functions:
$$\tilde{\Ub}^{(l)} = \Ub^{(l)}, \quad \tilde{\Vb}^{(l)} = \Vb^{(l)}, \quad \tilde{\bb}^{(l)} = \bb^{(l)}, \quad \tilde{\hb}_0^{(l)} = \hb_0^{(l)} \quad \forall l \in [L]$$
$$\implies \tilde{\hb}_t^{(l)} = \hb_t^{(l)} \quad \forall t \text{ and } \forall l \in [L].$$
Then, we set the last layer $L+1$, such that it computes the identity over the hidden vectors of the previous layer $\tilde{\hb}_t^{(L)}$:
$$\tilde{\Ub}^{(L+1)} = \Ib_{n}, \quad \tilde{\Vb}^{(L+1)} = \mathbf{0}, \quad \tilde{\bb}^{(L+1)} = \mathbf{0}, \quad \tilde{\hb}_0^{(L+1)} = \mathbf{0}$$
$$\implies \tilde{\hb}_t^{(L+1)} = \Ib_{n} \; \tilde{\hb}_t^{(L)} + \mathbf{0} \; \tilde{\hb}_{t-1}^{(L+1)} + \mathbf{0} = \tilde{\hb}_t^{(L)}= \hb_t^{(L)}.$$
This parameterization is such that $\tilde{\hb}_t^{(L+1)} =\hb_t^{(L)} $ for all $t$, which implies that $\tilde{h} = h$ and proves the inclusion $\hRNN{n,L} \subseteq \hRNN{n,L+1}$. Note that this proof also works for nonlinear RNNs by applying the same principles: using the same nonlinearities for the first $L$ layers and leaving the last layer linear\footnote{Even in the case where the last layer has to be non-linear, one could find parameters such that the last layer approximates the identity operator to an arbitrary precision~(assuming bounded domain of the inputs and infinite precision to represent floating operations).}. 

~\paragraph{Strict inclusion: } 
To prove strict inclusion of this theorem, we consider the~(real-valued) function $f_p$ defined for any $p$ by
\begin{equation}
    f_p(\xb_1,\xb_2,\cdots, \xb_t) = \begin{cases}
    (\xb_{t-p})_1&\text{ if } t-p > 0\\
    0&\text{ otherwise.}
\end{cases}
\tag{\ref{eq:fp}} 
\end{equation}

Informally, the function $f_p$ copies the first component of $\xb_t$ $p$ time steps forward. To simplify the notation, let $x_t = (\xb_t)_1$ be the first component of the input at time $t$. We first show in the following lemma that for any $n$ and any $p$ the function $f_p$ can be computed by a linear RNN as soon as it is deep enough~(at least $p/(n-1)$ layers).
\begin{lemma}\label{lemma:fpconstruction}
For any $n>1$ and $p\geq 1$, $f_p \in \linearmaps{n, 1}  \circ \hRNN{n,\ceil{p/(n-1)}}$.
\end{lemma}
\begin{proof}
The lemma is proved by construction.  We consider an RNN computing a function $h$ and a linear mapping $w$ such that $w \circ h = f_p$.  In this RNN, each hidden vectors is the concatenation of the first dimension of $n$ consecutive inputs. At each time step, the input occupying the first units of the hidden vector jumps to the next layer while the others are sent to the next (in time) hidden state through the recurrent connection: 

 \[\hb^{(1)}_t = \begin{bmatrix} x_{t-(n-1)} \\ \vdots \\ x_{t} \end{bmatrix}, \quad 
 \hb^{(2)}_t = \begin{bmatrix} x_{t-2(n-1)} \\ \vdots \\ x_{t-(n-1)} \end{bmatrix}, \quad \dots \; , \quad
 \hb^{(L)}_t = \begin{bmatrix} x_{t-L(n-1)} \\ \vdots \\ x_{t-(L-1)(n-1)} \end{bmatrix}\]

Since the depth of the network is $L=\ceil{\frac{p}{n-1}}$, the component $1+ L(n-1)-p$ of $\hb^{(L)}_t$ contains $x_{t-p}$. Moreover, it is ensured that, at all time, the first dimension of the previous $p$ inputs required for future computations are contained within the $L$ hidden vectors, see Figure \ref{fig:sketch_constrution_appendix}. 
 \begin{figure}[h]
\centering
\includegraphics[width=0.6\textwidth]{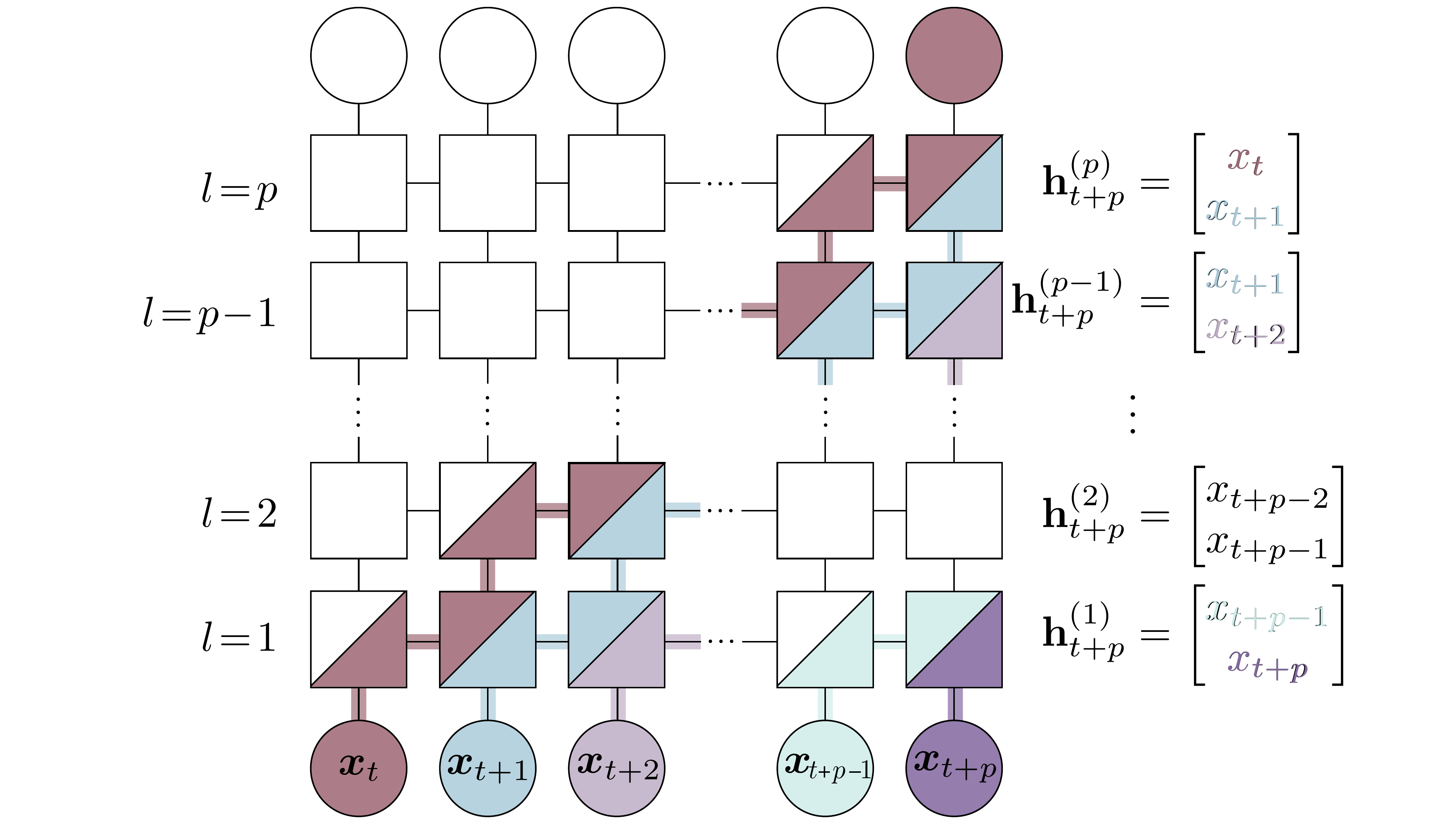}
\caption{Schematic representation of RNN's computation of $f_p$ for $n=2$ and $L=p$.}\label{fig:sketch_constrution_appendix}
\end{figure}

This computation is achieved with the following weight matrices, null biases and null initial hidden vectors :
\[\Ub^{(1)} = \begin{bmatrix} 0 \\ \vdots \\ 0 \\1 \end{bmatrix}, \quad 
\Ub^{(l)} = \begin{bNiceArray}{c|ccc}[margin] 0 & \Block{2-3}{ \mathbf{0} }  && \\ \vdots && \\ \hline 1 & 0 & \cdots & 0 \end{bNiceArray} \; \forall l\in [2,L] \quad \text{ and } \quad
\Vb^{(l)} = \begin{bNiceArray}{c|ccc}[margin] 0 & \Block{2-3}{ \Ib_{n-1} }  && \\ \vdots && \\ \hline 0 & 0 &\cdots & 0 \end{bNiceArray} \; \forall l \in [L].\]

One can verify that by composing $h$ with the linear map $w \in \linearmaps{n,1}$ defined by $w(\hb)=[\hb]_1$, we obtain $f_p$, which verifies the statement of the lemma.  
\end{proof}
We now show in the following lemma how the number of layers and hidden neurons needed for an RNN to compute the function $f_p$ are lower bounded by the lag value $p$.
\begin{lemma}\label{lemma:fpbound}
If $f_p \in \linearmaps{n, 1}  \circ \hRNN{n,L}$, then $p \leq L(n-1)$. \label{eq:boundP_d1}
\end{lemma}

\begin{proof}
 In order to prove the lemma, we formalize the computation of $f_p$ by an RNN as a flow graph in which information is propagated.
The graph is a grid where the horizontal axis correspond to the time $t$ and the vertical axis to the depth $l$. There is a total of $L+2$ layers: one for the input ($l=0$), one for the output ($l=L+1$) and $L$ hidden layers in between. Each node can be identified by its layer $l$ and time step $t$. Given the recurrent nature of the function, information can only flow through depth from $l$ to $l+1$ or through the recurrent connection from $t$ to $t+1$. 

In order for an RNN to compute $f_p$, for each time step $i$, the information of $x_i$ needs to flow through the graph  up to $\hb_{i+p}^{(L)}$ in the last layer $p$ time steps later. Thus, there exists a path in the graph between $\hb_i^{(1)}$ and $\hb_{i+p}^{(L)}$ through which the information of $x_i$ is propagated. Each step in this path is either horizontal or vertical, meaning either the information is propagated through recurrence, from $\hb_t^{(l)}$ to $\hb_{t+1}^{(l)}$, or through depth, from $\hb_t^{(l)}$ to $\hb_{t}^{(l+1)}$. See Figure~\ref{fig:conservation_computation_capacity}~(center) for an illustration. 
Note that we will use both $t$ and $i$ to denote time steps, but we will use $i$s for information flowing through the network and $t$s for the actual time steps of the computation. 

Now observe that for the information of $x_i$ to have been propagated through this path without loss of information, $x_i$ must have appeared as one of the components~(neuron) of each of the hidden state appearing on this path\footnote{One may argue that information about $x_i$ could have been separated between 2 or more neurons, but this would result in a sub-optimal construction since more than 2 neurons would have been used to store what could more efficiently be stored in a single neuron. }. We now introduce a notation to formalize the flow of information through the computation network of the RNN.

\textit{Definition: }
For each $i,t\geq 1$ and $1\leq l < L$, we let $\fr{i,t,l} = 1$ if the information of $x_i$ flows from $t$ to $t+1$ at layer $l$, and $0$ otherwise; similarly, we let  $\fu{i,t,l} = 1$ if the information of $x_i$ flows  from $l$ to $l+1$ at time $t$, and $0$ otherwise. \\
    
This definition is illustrated in Figure~\ref{fig:conservation_computation_capacity}~(left).
The computation $f_p$ is schematized in the center part of the figure:  the information of  $x_i$ introduced at time $t=i$ must flow out from the last layer of the RNN as the output at time $t=i+p$.

Since it is introduced at time $t=i$, $x_i$ cannot flow out of the input at other times or flow horizontally from earlier times.

\begin{figure}[h]
    \centering
    \includegraphics[width=0.75\textwidth]{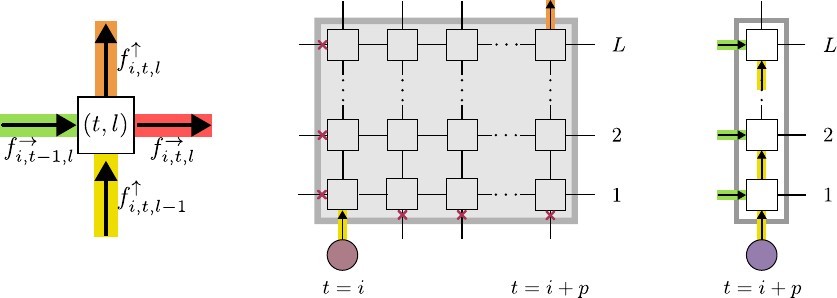}
    \caption{Left panel: Information flow through $\hb^{(l)}_t$. Center panel: Schematization of computation $f_p$. Right panel: The information capacity needed for all the hidden states across the layers at some time $t$ is obtained by considering the sum of information flows that enter the nodes. \label{fig:conservation_computation_capacity}}
\end{figure}

Lemma \ref{eq:boundP_d1} is derived by inspecting the total information capacity needed across all layers at time $t$, which is constrained by the number of neurons in the hidden states. At each time step $t$, the sum of all information flowing in all the neurons in the hidden states cannot exceed the total size of the hidden states $n L$~(see the right panel in Figure~\ref{fig:conservation_computation_capacity}). Formally, for all $t\geq 1$, we have
\begin{equation}
nL \geq \sum_{l=1}^{L}\sum_{i=1}^{t-1}f_{i,t-1,l}^{\rightarrow}+\sum_{l=1}^{L}\sum_{i=1}^{t}f_{i,t,l-1}^{\uparrow} .\label{eq:proto_boundP_d1}
\end{equation}
The first part of the r.h.s. corresponds to the information that is flowing from the previous time step through recurrence (at the same layer), and the second one corresponds to the information flowing between two consecutive layers through depth (at time step $t$).

The first double sum in Eq.~\eqref{eq:proto_boundP_d1} is simplified by observing that to compute $f_p$, all input information $x_i$ where $i\in[t-p,t-1]$ must flow from $t-1$ to $t$. To achieve this, each input must flow horizontally in (at least) one of the layers, meaning that $\sum_{l=1}^L \fr{i,t-1,l} \geq 1$ for all $i\in[t-p,t-1]$. 
Hence, the first double sum can be lower bounded as $\sum_{l=1}^{L}\sum_{i=1}^{t-1}\fr{i,t-1,l} \geq p$, leading to 
\begin{equation}
nL \geq p +\sum_{l=1}^{L}\sum_{i=1}^{t}f_{i,t,l-1}^{\uparrow} \label{eq:proto_boundP_d2}
\end{equation}
for all time step $t \geq 1$. Since this inequality is true for any time step $t$, we can sum up these inequalities for each time step up to an arbitrary time step $K$ to obtain
\begin{equation}
KnL \geq Kp + \sum_{t=1}^K\sum_{l=1}^{L}\sum_{i=1}^{t}f_{i,t,l-1}^{\uparrow}\ ,\ \text{ for any } K \label{eq:proto_boundP_d3}.
\end{equation}

To lower bound the last term, first observe that to make its way from the input layer $l=0$ to the last layer $l=L$, a given information $x_i$ must flow vertically at least once in each layer between the input and output times $t\in [i,i+p]$. Formally,  we must have that, for each time step $i$, $\sum_{t=i}^{i+p}f_{i,t,l-1}^{\uparrow} \geq 1 $ for all $l \in [L]$, hence $\sum_{l=1}^L\sum_{t=i}^{i+p}f_{i,t,l-1}^{\uparrow} \geq L $. Again, since this inequality holds for each time step $i$, we can sum up inequalities up to an arbitrary time step. In particular, we have 
\[\sum_{i=1}^{K-p}\sum_{l=1}^L\sum_{t=i}^{i+p}f_{i,t,l-1}^{\uparrow} \geq (K-p)L \ ,\ \text{ for any } K \geq p.\]

Now, we can show that 
\[\sum_{t=1}^K\sum_{l=1}^{L}\sum_{i=1}^{t}f_{i,t,l-1}^{\uparrow} \geq \sum_{i=1}^{K-p}\sum_{l=1}^L\sum_{t=i}^{i+p}f_{i,t,l-1}^{\uparrow}. \] Indeed, every term appearing in the r.h.s. is of the form $f_{\tilde i,\tilde i+\tau,l-1}^{\uparrow}$ with $\tilde i\in[1,K-p]$ and $\tau \in [0,p]$, and thus appears in the l.h.s. as well when $t=\tilde i + \tau$~(since in this case $\tilde i$ is the range of the sum over $i$ of the l.h.s., which is $[1,t] = [1,\tilde i + \tau]$). 

Finally, we thus have that $\sum_{t=1}^K\sum_{l=1}^{L}\sum_{i=1}^{t}f_{i,t,l-1}^{\uparrow} \geq (K-p)L $, which allows us to obtain from Eq.~\eqref{eq:proto_boundP_d3} the final inequality 
\[KnL \geq Kp + (K-p)L\] 
which can be re-arranged as
 \[p \leq L(n-1)  + \frac{pL}{K} .\] 
Since this equality is true for any $K$, we can choose $K$ large enough for $\frac{pL}{K}$ to be strictly smaller than one,  which implies the result, $p \leq L(n-1)$. 

\end{proof}

To prove the theorem, it only remains to combine the two lemmas. Let $n>1$ and $L\geq 1$. We choose $p=(L+1)(n-1)$; from Lemma~\ref{lemma:fpconstruction}, $f_p \in \linearmaps{n, 1}  \circ \hRNN{n,L+1}$, from Lemma~\ref{lemma:fpbound},  $f_p \not\in \linearmaps{n, 1}  \circ \hRNN{n,L}$. This shows that $\hRNN{n,L} \subsetneq \hRNN{n,L+1}$.

\end{proof}

\subsection{Formalization of Proposition~\ref{thm:(n,l)vs(nl,1)}}\label{apx:prop(n,l)vs(nl,1)}
\begin{proposition*}
    For $n>1$ and $L>1$, a single-layer linear RNN of hidden size $nL$ has a greater latent representational capacity than a linear RNN of size $n$ and depth $L$. 
\end{proposition*}
The following theorem formalizes the previous proposition. 
\begin{theorem}
Let $\hRNNconcat{n,l}$ denote the set of functions $h^{(\|)}$ mapping an input sequence to the concatenation of hidden vectors computed at each layers of an RNN of hidden size $n$ and depth $L$:
\[h^{(\|)}(\xb_1, \dots, \xb_T) = (\hb_1^{(\|)}, \dots, \hb_T^{(\|)}) \quad \text{ with } \quad
\hb_t^{(\|)} = \begin{bmatrix} \hb_t^{(1)} \\ \vdots \\ \hb_t^{(L)} \end{bmatrix}.\]

For any $n>1$ and $L\geq 1$, $\hRNNconcat{n,L} \subsetneq \hRNN{nL, 1}$, for linear RNNs.
\end{theorem}

\begin{proof}
~\paragraph{Inclusion: } 
For any $h^{(\|)} \in \hRNNconcat{n,L}$, we show there exists a $\tilde{h} \in \hRNN{nL,1}$ such that $\tilde{h}=h^{(\|)}$, that is $\hRNNconcat{n,L} \subseteq \hRNN{nL, 1}$.
The idea is for the hidden vectors $\tilde{\hb}_t$ of the function $\tilde{h}$ to be composed of $L$ blocks of $n$ dimensions. Each block reproduces the computation of one of the layer hidden vector $\hb_t^{(l)}$ that is concatenated in the function $h^{(\|)}$. The parameters of $\tilde{h}$ are thus also separated in blocks as follow:
\begin{equation}
\label{eq:blockparam.htilde}
\tilde{\hb}_t = \tilde{\Ub} \xb_t + \tilde{\Vb} \tilde{\hb}_{t-1} + \tilde{\bb} \to \begin{bmatrix} \tilde{\hb}_{t,1} \\ \vdots \\ \tilde{\hb}_{t,L} \end{bmatrix} =  \begin{bmatrix} \tilde{\Ub}_{1} \\ \vdots \\ \tilde{\Ub}_{L} \end{bmatrix} \xb_t + 
\begin{bmatrix} \tilde{\Vb}_{11} & \dots &  \tilde{\Vb}_{1L}  
\\ \vdots & \ddots & \vdots
\\ \tilde{\Vb}_{L1} & \dots & \tilde{\Vb}_{LL} 
\end{bmatrix}
\begin{bmatrix} \tilde{\hb}_{t-1, 1} \\ \vdots \\ \tilde{\hb}_{t-1, L} \end{bmatrix} 
+ \begin{bmatrix} \tilde{\bb}_{1} \\ \vdots \\ \tilde{\bb}_{L} \end{bmatrix}
\end{equation}
with $\tilde{\hb}_{t,l} \in \Rbb^{n}$, $\tilde{\Ub_l} \in \Rbb^{n \times d}$, $\tilde{\Vb}_{lj} \in \Rbb^{n \times n}$ and $\tilde{\bb}_l \in \Rbb^n$ for all $l \in [L]$.

For each $l\in[L]$, we want the computation $\tilde{\hb}_{t,l} = \tilde{\Ub}_l \xb_t + \sum_{j=1}^L \tilde{\Vb}_{lj} \tilde{\hb}_{t-1, j} + \tilde{\bb}_l$ to correspond to $\hb_t^{(l)} = \Ub^{(l)} \hb_{t}^{(l-1)} + \Vb^{(l)} \hb_{t-1}^{(l)} + \bb^{(l)}$.

As shown in Corollary~\ref{cor:induction-RNN}, $\hb_t^{(l)}$ can be expressed in terms of $\xb_t$ and $\hb_{t-1}^{(\bar{l})}$ with $\bar{l} \in[l]$ as : 
\[
\hb_t^{(l)} = \left(\prod_{i=1}^{l}\Ub^{(i)}\right)\xb_{t}+\sum_{j=1}^{l}\left(\prod_{k=j+1}^{l}\Ub^{(k)}\right)\Vb^{(j)}\hb_{t-1}^{(j)}+\sum_{j=1}^{l}\left(\prod_{k=j+1}^{l}\Ub^{(k)}\right)\bb^{(j)}
\]

It then suffices to set the parameters of an RNN computing $\tilde{h}$ as
\[ 
\tilde{\Ub}_l = \prod^{l}_{i=1} \Ub^{(i)},\quad 
\tilde{\Vb}_{lj} = \begin{cases} \left(\prod^{l}_{k=j+1} \Ub^{(k)}\right) \Vb^{(j)} &  j \leq l \\0  & j>l \end{cases}, \quad    \tilde{\bb}_{l} = \sum_{j} \left(\prod^{l}_{k=j+1} \Ub^{(k)}\right) \bb^{(j)}
\]
in order to have $\tilde{\hb}_{t,l} = \hb_t^{(l)}$ for all $l\in[L]$, and therefore $\tilde{h} = h^{(\|)}$. This concludes the proof for $\hRNNconcat{n,L} \subseteq \hRNN{nL, 1}$.
~\paragraph{Strict inclusion: }
We show there exists a function $\tilde{h} \in \hRNN{nL, 1}$ that cannot be computed by any function in $h \in \hRNNconcat{n,L}$, that is $\hRNNconcat{n,L} \not \supset \hRNN{nL, 1}$.

Consider the function  $\tilde{h} \in \hRNN{nL, 1}$ computed by an RNN parameterized with $\tilde{\Vb}= \mathbf{0}$, $\tilde{\bb} = \mathbf{0}$
and $\tilde{\Ub}$ block parameterized as in Eq.~\ref{eq:blockparam.htilde} with $\rank(\tilde{\Ub}_1) = 1$ and $\rank(\tilde{\Ub}_l) = 2$ for all $l \in [2,l]$\footnote{In fact, the result works for any parameterization such that $\rank(\tilde{\Ub}_1) = 1$ while $\rank(\tilde{\Ub}_l) > 1$ for at least one $l \in [2,l]$}.
Suppose there exists $h \in \hRNNconcat{n,L}$ such that $h = \tilde{h}$. Then, we would have equality between the first hidden vectors $\tilde{\hb}_1$ and $\hb_1^{(\|)}$ which implies on one hand, 
\begin{align*}
    \tilde{\hb}_{1,1} = \hb_1^{(1)} &\implies \tilde{\Ub}_1 \xb_1 = \Ub^{(1)} \xb_1 \implies \tilde{\Ub}_1 = \Ub^{(1)} \\
    &\implies \rank(\Ub^{(1)}) = \rank(\tilde{\Ub}_1)=1,
\end{align*}

but on the other hand, 
\begin{align*}
    \tilde{\hb}_{1,2} = \hb_1^{(2)} &\implies \tilde{\Ub}_2 \xb_1 = \Ub^{(2)} \Ub^{(1)} \xb_1 \implies \tilde{\Ub}_2 = \Ub^{(2)} \Ub^{(1)} \\
    &\implies \rank(\Ub^{(2)} \Ub^{(1)}) = \rank(\tilde{\Ub}_2)=2\\
    &\implies \rank( \Ub^{(1)}) \geq 2.
\end{align*}

The existence of $h \in \hRNNconcat{n,L}$ such that $h = \tilde{h}$ leads to a contradiction, proving that $h \not \in \hRNN{nL, 1}$ and $\hRNNconcat{n,L} \not \supset \hRNN{nL, 1}$.
\end{proof}
\subsection{Proof of Theorem~\ref{thm:params}}
\begin{theorem*}
Considering linear RNNs with input dimension $d=1$, let $\params(n,L)$ denote the number of parameters of a $L$-layers RNN with hidden size $n$ \footnote{One can check that $\params(n,L)=(2L-1)n^2+(L+1)n$~(excluding initial hidden states as parameters count, though including them would not change the result).}.

For any depth $L$ and for any hidden size $n\geq 4$, there exists a function $f \in \hRNN{n,L}$ such that, for all $\tilde{L} < L$ and $\tilde{n}$, if $f \in \hRNN{\tilde{n}, \tilde{L}}$ then $\params(\tilde{n}, \tilde{L})>\params(n,L)$.
\end{theorem*}
\begin{proof}
Let $L>1$ and $n>1$. We set $p=L(n-1)$ and show that the function $f_p$ defined in Eq.~\eqref{eq:fp} satisfies the result. From Lemma~\ref{lemma:fpconstruction}, we know that $f_p\in\linearmaps{n,1} \circ \hRNN{n,L}$. Now, let $\tilde L $ be an integer such that $1\leq \tilde L < L$ and  $f_p\in \linearmaps{\tilde n,1} \circ \hRNN{\tilde n,\tilde L}$ for some integer $\tilde n$. Then, from Lemma~\ref{lemma:fpbound}, we must have $p\leq \tilde L (\tilde n - 1)$, i.e., 
\[
L(n-1) \leq \tilde L (\tilde n-1).
\]
It follows that 
\[
\tilde n \geq \frac{L}{\tilde L}(n-1) + 1.
\]
Since $\params(n, L) = (2L-1)n^2+(L+1)n$ is monotonous increasing in $n$ for $n\geq 1, L\geq 1$, a lower bound can be obtained as
\begin{align*}
    \params(\tilde n , \tilde L) 
    \geq \params\left(\frac{L}{\tilde L}(n-1) + 1 , \tilde L\right) = (2\tilde L-1)\left(\frac{L}{\tilde L}\right)^2n^2+bn+c
\end{align*}
where $b$ and $c$ are constants that depend only on $L$ and $\tilde L$. It follows that 
    \begin{equation}\label{eq:poly.params.lower.bound}
\params(\tilde n, \tilde L) - \params(n, L) 
\geq 
\left((2\tilde L-1)\left(\frac{L}{\tilde L}\right)^2-2L + 1\right) n^2 + b^\prime n + c
\end{equation}
where $b^\prime$ does not depend on $n$. 

The r.h.s. of Eq.~\eqref{eq:poly.params.lower.bound} is a second order polynomial in $n$ and its leading coefficient is positive. To see this, this coefficient can be developped and re-arranged as
\begin{equation}\label{eq:a_positive}
a  = \frac{(L-\tilde{L})(2\tilde L L - L - \tilde L)}{\tilde{L}^2} > 0.
\end{equation}
All factors  $\tilde{L}^{-2}$,  $L-\tilde L$ and  $2\tilde LL - (\tilde L + L) $ are positive for $1\leq\tilde L < L$, showing that the leading coefficient is indeed positive. Since the r.h.s. of Eq.~\eqref{eq:poly.params.lower.bound} is a second order polynomial with a strictly positive leading term~(its graph is an upward-opening parabola), there exists an $N$ for which this polynomial is positive for any $n > N$, i.e., $\params(\tilde n, \tilde L) > \params(n, L) $ for all $n > N$.
\end{proof}

\subsubsection{Critical $N$}
It is in fact possible to find the critical value $N$ where the r.h.s. of Eq.~\eqref{eq:poly.params.lower.bound} vanishes, that is we solve $a n^2+b'n+c=0$. The leading coefficient $a$ is shown in Eq.~\eqref{eq:a_positive}, and other coefficients can be found easily. The coefficients we obtain have common factors and the equation can be factorized as follows: $(L-\tilde{L})\tilde{L}^{-2} (\tilde a n^2+ \tilde b n+\tilde c)=0$, where 
\[
\tilde a = 2 L \tilde L - L - \tilde L, \quad 
\tilde b = -2\tilde{a} - \tilde{L}, \quad 
\tilde c = \tilde{a} - \tilde{L}(3 \tilde{L}-1).
\]
It is sufficient to solve $\tilde a n^2+ \tilde b n+\tilde c = 0$. The largest root is given as a function of $L,\tilde L$
\[
    n_0^+(L,\tilde L) = 1 + \frac{\tilde{L}(1+\sqrt{12 \tilde a +1})}{2 \tilde a}.
\]
where $\tilde a$'s dependence in $L$ and $\tilde L$ is not noted explicitly to keep the notation compact.
The maximal $n_0^+(L,\tilde L)$ among pairs $(L,\tilde L)$ respecting $1\leq\tilde L < L$ gives $N$. We observe that the gradient in $L$ is always negative
\[
    \frac{\partial n_0^+}{\partial L} = -\frac{\left(6 \tilde{a} +\sqrt{12 \tilde{a} +1}+1\right) \tilde{L} \left(2 \tilde{L}-1\right)}{2 \tilde{a} ^2 \sqrt{12 \tilde{a} +1}}<0
\]
which is easy to see by noting that $\tilde{a}>0$ and $\tilde{L} \geq 1$. This implies that $n_0^+$ is maximized by setting $L$ to its smallest possible value,  $\tilde L + 1$
\[  
    n_0^+(\tilde L+1,\tilde L)  = 1 + \frac{\tilde{L} \left(1+\sqrt{24 \tilde{L}^2-11}\right)}{4 \tilde{L}^2-2}.
\]
This expression decreases monotonically in $\tilde L$ for $\tilde L \geq 1$. Thus, $n_0^+$ is again maximized by choosing the appropriate boundary, $\tilde{L}=1$. The maximum is  $n_0^+(L=2, \tilde{L}=1)= (3+\sqrt{13})/2 \approx 3.30$, which is rounded up to $4$. Therefore, for any depth $L$ and any hidden size $n \geq 4$, there exists a function computable by an RNN with $L$ layers and $n$ neurons which cannot be computed by a shallower network without increasing the number of parameters.

\subsection{Formalization of "linear BIRNNs compute polynomial transformations"}\label{apx:proplinearBIRNN}
\begin{proposition*}
Let $h\in \hBIRNN{n,L}$ be the function computed by a BIRNN. Then, for any $T$, the function $f : \Rbb^{dT} \to \Rbb^{nT}$ mapping~(concatenations of) sequences of inputs $(\xb_1,\cdots,\xb_T)$ to~(concatenations of) sequences of hidden states $(\hb^{(L)}_1,\cdots, \hb^{(L)}_T)$ is a multivariate polynomial whose order is  at most $T^L$.
\end{proposition*}

\begin{proof}
First, it is easy to check that $f$ is a polynomial since the update function of a linear 2RNN is a polynomial map of order 2:
$$\vec{h}_t^{(l)} = \At^{(l)} \times_1 \hb_{t-1}^{(l)}\times_2 \hb_t^{(l-1)} + \Vb^{(l)} \vec{h}_{t-1}^{(l)} + \Ub^{(l)} \vec{h}_t^{(l-1)} + \bb^{(l)}. $$
Thus, each $\vec{h}_t^{(l)}$ is obtained by composing polynomial maps of the inputs, and $f$ is itself a polynomial map.

We now proceed to show that the order of the polynomial is at most $T^L$ by considering the recurrence on $L$. For sake of simplicity, we assume that only the second order parameters $\At^{(l)}$  of the 2RNN are non-zero, since the polynomial interactions of highest order in $\vec{h}_t^{(l)}$ are the second order ones. 
For $L=1$, for any $t$, we have 
$$\vec{h}_t^{(1)} = \Tten \times_1 \hb^{(1)}_0 \times_2 \xb_1 \times_3 \cdots \times_{t+1} \xb_t$$
where the $(t+2)$th order tensor $\Tten \in \Rbb^{n\times d\times \cdots \times d \times n}$ is defined by 
$$\Tten_{i,j_1,\cdots, j_t, k} = \sum_{r_1=1}^{n}  \sum_{r_2=1}^{n} \cdots \sum_{r_2=1}^{n}  (\hb^{(1)}_0)_{i} \At^{(1)}_{i,j_1,r_1}
\At^{(1)}_{r_1,j_2,r_2}\At^{(1)}_{r_2,j_3,r_3}\cdots\At^{(1)}_{r_n,j_n,k}$$
for all $i,k\in [n],j_1,\cdots,j_n\in[d]$. This shows that, for $L=1$, $f$ is a polynomial map of order at most $t$. Indeed, recall that 
$$ (\Tten \times_1 \hb^{(1)}_0 \times_2 \xb_1 \times_3 \cdots \times_{t+1} \xb_t)_{k} = \sum_{i,j_1,\cdots,j_n} 
\Tten_{i,j_1,\cdots,j_n,k} ( \hb^{(1)}_0)_{i}(\xb_1)_{j_1}  \cdots  (\xb_t)_{j_t}.$$

Now, assuming the result true for all integers strictly less than $L$, the output of the $(L-1)$th layer, $(\hb^{(L-1)}_1,\cdots, \hb^{(L-1)}_T)$, is a polynomial map of the inputs of order at most $T^{L-1}$. Using the same argument as for the case $L=1$, it is straightforward to show that the output of the $L$th layer is a polynomial map of $(\hb^{(L-1)}_1,\cdots, \hb^{(L-1)}_T)$ of order at most $T$. Since the composition of two polynomial maps of orders $p_1$ and $p_2$, respectively, is at most $p_1p_2$, it follows that the output of the $L$th layer is a polynomial map of the inputs of order at most $T^{L-1}T = T^L$.
\end{proof}

\subsection{Proof of Theorem~\ref{thm:BIRNN}}\label{apx:thmBIRNN}
\begin{theorem*}
For $n>1$ and $L\geq1$, $\hBIRNN{n,L} \subsetneq \hBIRNN{n,L+1}$ for linear BIRNNs.
\end{theorem*}

\begin{proof}
~\paragraph{Inclusion: }
We begin by showing that for any $h \in \hBIRNN{n,L}$, there exists a function $\tilde{h} \in \hBIRNN{n, L+1}$ such that $h=\tilde{h}$, that is $\hBIRNN{n,L} \subseteq \hBIRNN{n,L+1}$.
Let the function $h$ be computed by a BIRNN parameterized by $\{\At^{(l)}, \hb_0^{(l)}\}_{l=1}^L$. We set the first $L$ layers of the BIRNN computing $\tilde{h}$ the same as the one computing $h$ such that $\tilde{\hb}_t^{(L)} = \hb_t^{(L)} \forall t$:
\[
\skew{5}\tilde{\At}^{(l)}=\At^{(l)}, \quad \tilde{\hb}_0^{(l)} = \hb_0^{(l)} \quad \forall l \in [L]. 
\]

Then, for the last layer $L+1$, it suffices to add residual connections from the $L$th layer  and set $\tilde{\At}_{ijk}=\mathbf{0} \text{ and }\tilde{\hb}_0^{(L+1)} = \mathbf{0}$ to have $\tilde{\hb}_t^{(L+1)} = \tilde{\hb}_t^{(L)} = \hb_t^{(L)} \forall t$ which implies $\tilde{h}=h$ and proves the inclusion $\hBIRNN{n,L} \subseteq \hBIRNN{n,L+1}$.

~\paragraph{Strict Inclusion: }
We show there exists a function $\tilde{h} \in \hBIRNN{n,L+1}$ that cannot be computed by any function $h \in \hBIRNN{n,L}$, that is $\hBIRNN{n,L} \not \supset \hBIRNN{n,L+1}$.
Let $poly(\xb,k)$ denote the set of functions having a polynomial dependency in $\xb$ of degree up to $k$, irrespective of dependencies in other variables. 
First, we prove by induction that for any $h \in \hBIRNN{n,L}$, the second time step output $\hb_2^{(L)}$ has a polynomial dependency in $\xb_1$ of degree at most $L$, i.e., $\hb_2^{(L)}\in poly(\xb_1,L)$~(using a slight abuse of notation):
\begin{itemize}[left=10pt]
    \item First, observe from Lemma \ref{lemma:induction-2RNN} that the hidden states in the first time step are given by
    \[ \hb_1^{(l)}  =\left( \prod_{i=1}^{l} ( \At^{(i)} \times_{1} \hb_{0}^{(i)})^{\top} \right)\xb_{1} \]
    which means that $\hb_1^{(l)} \in poly(\xb_1,1)\; \forall l$.
    \item In the second time step,  Lemma \ref{lemma:induction-2RNN} can similarly be used to obtain
    \[
        \hb_2^{(l)} = \left( \prod_{i=1}^{l}(\At^{(i)}\times_{1}\hb_{1}^{(i)})^{\top} \right) \xb_{2}.
    \]
    $\hb_2^{(l)}$ contains terms with a maximal number of $l$ factors in $\{\hb_1^{(i)}\}_{i=1}^{l}$, which, given  that $\hb_1^{(l)} \in poly(\xb_1,1)$, implies that $\hb_2^{(l)} \in poly(\xb_1,l)$.
\end{itemize}

Now we prove there exists a function $\tilde{h} \in \hBIRNN{n,L+1}$ such that $\tilde{\hb}_2^{(L+1)}=\diag(\xb_1)^{L+1}\xb_2$. We begin with the case $n=d$ and generalize after. Consider the function $\tilde{h}$ parameterized by $\{\tilde{\At}^{(l)} = \delta_{ijk}, \tilde{\hb}_0^{(l)}=[1, \dots, 1]^\top \}_{l=1}^L$. From Lemma \ref{lemma:induction-2RNN} and by observing that for any $\yb \in \Rbb^{n}$, $\tilde{\At}^{(l)} \times_1 \yb = \diag(\yb)$, we obtain:
\[ 
    \tilde{\hb}_t^{(l)}=(\prod_{i=1}^l \diag(\tilde{\hb}_{t-1}^{(i)})) \xb_t.
\]
Given the hidden states initialization, we have $\diag(\tilde{\hb}_0^{(l)}) = \Ib \quad \forall l$ and the hidden states in the first time step become $\tilde{\hb}_1^{(l)} = \xb_1$. Inserting in the expression for $\tilde{\hb}_2^{(l)}$, we obtain
\[ 
    \tilde{\hb}_2^{(l)}= \diag(\xb_1)^l \xb_2.
\]
This proves that $\tilde{h}$ is such that $\tilde{\hb}_2^{(L+1)}=\diag(\xb_1)^{L+1}\xb_2$. This function can not be computed by any $h \in \hBIRNN{n, L}$ as for any such $h$, $\hb_2^{(L)}\in poly(\xb_1,L)$.

In the case $n>d$, we set $\tilde{\At}^{(1)}_{ijk} = \delta_{ijk} \forall i,j,k \in [1,d]$ and  $\tilde{\At}_{ijk}^{(1)} = 0 \quad \forall i,k \in [d+1,n]$. The result will be the same, but with 0 padding on the extra dimensions.
In the case $n<d$, we set $\tilde{\At}^{(1)}_{ijk} = \delta_{ijk} \forall j \in [1,n]$ and $\tilde{\At}_{ijk}^{(1)} = 0 \quad \forall j \in [n+1,d]$. The result will be the same, but for the first $n$ dimensions of the inputs : $\hb_2^{(L+1)}=\diag([\xb_1]_{1:n})^{L+1}[\xb_2]_{1:n}.$

\end{proof}
\subsubsection{Theorem~\ref{thm:BIRNN} for 2RNNs}\label{apx:thm2RNN}
In the main text, we focused on BIRNNs to isolate the effect of bilinear interactions, but the same proof technique can be used to prove the result for 2RNNs. Here is the adaptation of Theorem~\ref{thm:BIRNN} for 2RNNs and its proof. 
\begin{theorem*}
For $n>1$ and $L\geq1$, $\hSORNN{n,L} \subsetneq \hSORNN{n,L+1}$ for linear BIRNNs.
\end{theorem*}

\begin{proof}
~\paragraph{Inclusion: }
We begin by showing that for any $h \in \hSORNN{n,L}$, there exists a function $\tilde{h} \in \hSORNN{n, L+1}$ such that $h=\tilde{h}$, that is $\hSORNN{n,L} \subseteq \hSORNN{n,L+1}$.
Let the function $h$ be computed by a 2RNN parameterized by $\{\At^{(l)}, \Ub^{(l)}, \Vb^{(l)}, \bb^{(l)}, \hb_0^{(l)}\}_{l=1}^L$. We set the first $L$ layers of the 2RNN computing $\tilde{h}$ the same as the one computing $h$ such that $\tilde{\hb}_t^{(L)} = \hb_t^{(L)} \forall t$:
\[
\tilde{\At}^{(l)}=\At^{(l)}, \quad \tilde{\Ub}^{(l)} = \Ub^{(l)}, \quad \tilde{\Vb}^{(l)} = \Vb^{(l)}, \quad \tilde{\bb}^{(l)} = \bb^{(l)}, \quad \tilde{\hb}_0^{(l)} = \hb_0^{(l)} \quad \forall l \in [L]. 
\]
Then, the last layer $L+1$ is set to perform the identity on $\tilde{\hb}^{(L)}_t$:
\[
\tilde{\At}=\mathbf{0}, \quad \tilde{\Ub}^{(L+1)} = \Ib_{n}, \quad \tilde{\Vb}^{(L+1)} = \mathbf{0}, \quad \tilde{\bb}^{(L+1)} = \mathbf{0}, \quad \tilde{\hb}_0^{(L+1)} = \mathbf{0}
\]
We thus have $\tilde{\hb}_t^{(L+1)} = \tilde{\hb}_t^{(L)} = \hb_t^{(L)} \forall t$ which implies $\tilde{h}=h$ and proves the inclusion $\hSORNN{n,L} \subseteq \hSORNN{n,L+1}$.
~\paragraph{Strict Inclusion: }
We show there exists a function $\tilde{h} \in \hSORNN{n,L+1}$ that cannot be computed by any function $h \in \hSORNN{n,L}$, that is $\hSORNN{n,L} \not \supset \hSORNN{n,L+1}$.
Let $poly(\xb,k)$ denote the set of functions having a polynomial dependency in $\xb$ of degree up to $k$, irrespective of dependencies in other variables. 
First, we prove by induction that for any $h \in \hSORNN{n,L}$, the second time step output $\hb_2^{(L)}$ has a polynomial dependency in $\xb_1$ of degree at most $L$, i.e., $\hb_2^{(L)}\in poly(\xb_1,L)$~(using a slight abuse of notation):
\begin{itemize}[left=10pt]
    \item First, observe from Lemma \ref{lemma:induction-2RNN} that the hidden states in the first time step are given by
    \[ \hb_1^{(l)}  =(\prod_{i=1}^{l}(\At^{(i)}\times_{1}\hb_{0}^{(i)})^{\top}+\Ub^{(i)})\xb_{1}+\sum_{j=1}^{l}(\prod_{k=j+1}^{l}(\At^{(k)}\times_{1}\hb_{0}^{(k)})^{\top}+\Ub^{(k)})(\Vb^{(j)}\hb_{0}^{(j)}+\bb^{(j)}) \]
    which means that $\hb_1^{(l)} \in poly(\xb_1,1)\; \forall l$.
    \item In the second time step,  Lemma \ref{lemma:induction-2RNN} can similarly be used to obtain
    \[
        \hb_2^{(l)} = (\prod_{i=1}^{l}(\At^{(i)}\times_{1}\hb_{1}^{(i)})^{\top}+\Ub^{(i)})\xb_{2}+\sum_{j=1}^{l}(\prod_{k=j+1}^{l}(\At^{(k)}\times_{1}\hb_{1}^{(k)})^{\top}+\Ub^{(l)})(\Vb^{(j)}\hb_{1}^{(j)}+\bb^{(j)}).
    \]
    $\hb_2^{(l)}$ contains terms with a maximal number of $l$ factors in $\{\hb_1^{(i)}\}_{i=1}^{l}$, which, given  that $\hb_1^{(l)} \in poly(\xb_1,1)$, implies that $\hb_2^{(l)} \in poly(\xb_1,l)$.
\end{itemize}

Now we prove there exists a function $\tilde{h} \in \hSORNN{n,L+1}$ such that $\tilde{\hb}_2^{(L+1)}=\diag(\xb_1)^{L+1}\xb_2$. We begin with the case $n=d$ and generalize after. Consider the function $\tilde{h}$ parameterized by $\{\tilde{\At}^{(l)} = \delta_{ijk},  \tilde{\Ub}^{(l)} = \mathbf{0}, \tilde{\Vb}^{(l)}=\mathbf{0}, \tilde{\bb}^{(l)}=\mathbf{0}, \tilde{\hb}_0^{(l)}=[1, \dots, 1]^\top \}_{l=1}^L$. By keeping only $\tilde{\At}^{(l)}$ in Lemma \ref{lemma:induction-2RNN}, we obtain that $\tilde{\hb}_t^{(l)}=(\prod_{i=1}^l \tilde{\At}^{(i)} \times_1 \tilde{\hb}_{t-1}^{(i)})^\top \xb_t$. By observing that for any $\yb \in \Rbb^{n}$, $\tilde{\At}^{(l)} \times_1 \yb = \diag(\yb)$, this is simplified to 
\[ 
    \tilde{\hb}_t^{(l)}=(\prod_{i=1}^l \diag(\tilde{\hb}_{t-1}^{(i)})) \xb_t.
\]
Given the hidden states initialization, we have $\diag(\tilde{\hb}_0^{(l)}) = \Ib \quad \forall l$ and the hidden states in the first time step become $\tilde{\hb}_1^{(l)} = \xb_1$. Inserting in the expression for $\tilde{\hb}_2^{(l)}$, we obtain
\[ 
    \tilde{\hb}_2^{(l)}= \diag(\xb_1)^l \xb_2.
\]
This proves that $\tilde{h}$ is such that $\tilde{\hb}_2^{(L+1)}=\diag(\xb_1)^{L+1}\xb_2$. This function can not be computed by any $h \in \hSORNN{n, L}$ as for any such $h$, $\hb_2^{(L)}\in poly(\xb_1,L)$.

In the case $n>d$, we set $\tilde{\At}^{(1)}_{ijk} = \delta_{ijk} \forall i,j,k \in [1,d]$ and  $\tilde{\At}_{ijk}^{(1)} = 0 \quad \forall i,k \in [d+1,n]$. The result will be the same, but with 0 padding on the extra dimensions.
In the case $n<d$, we set $\tilde{\At}^{(1)}_{ijk} = \delta_{ijk} \forall j \in [1,n]$ and $\tilde{\At}_{ijk}^{(1)} = 0 \quad \forall j \in [n+1,d]$. The result will be the same, but for the first $n$ dimensions of the inputs : $\hb_2^{(L+1)}=\diag([\xb_1]_{1:n})^{L+1}[\xb_2]_{1:n}.$

\end{proof}

\subsection{Proof of Corollary~\ref{cor:cprnn}}\label{apx:cor-cprnn}
The same proof technique presented in Appendix~\ref{apx:thmBIRNN} for Theorem~\ref{thm:BIRNN} can be applied to CPBIRNNs. It suffices to replace the parameter weight tensor $\At$ by the CP decomposition matrices $\CP{\Ab, \Bb, \Cb}$ (the formal definition for CP(BI)RNN can be found in Appendix~\ref{apx:cprnn}). 
\begin{corollary*}
    For $n>1$, $L\geq1$ and any $\tilde{R}\geq R >1$, $\hCPBIRNN{n,L,R} \subsetneq \hCPBIRNN{n,L+1, \tilde{R}}$ for linear CPBIRNNs.
\end{corollary*}

\begin{proof}
~\paragraph{Inclusion: }
We begin by showing that for any $h \in \hCPBIRNN{n,L,R}$, there exists a function $\tilde{h} \in \hCPBIRNN{n, L+1,R}$ such that $h=\tilde{h}$, that is $\hCPBIRNN{n,L,R} \subseteq \hCPBIRNN{n,L+1,R}$.
Let the function $h$ be computed by a CPBIRNN parameterized by $\{\Ab^{(l)}, \Bb^{(l)}, \Cb^{(l)}, \hb_0^{(l)}\}_{l=1}^L$. We set the first $L$ layers of the CPBIRNN computing $\tilde{h}$ the same as the one computing $h$ such that $\tilde{\hb}_t^{(L)} = \hb_t^{(L)} \forall t$:
\[
\tilde{\Ab}^{(l)}=\Ab^{(l)}, \quad \tilde{\Bb}^{(l)} = \Bb^{(l)}, \quad \tilde{\Cb}^{(l)} = \Cb^{(l)}, \quad \tilde{\hb}_0^{(l)} = \hb_0^{(l)} \quad \forall l \in [L]. 
\]

Then, for the last layer $L+1$, it suffices to add residual connections from the $L$th layer and set $\tilde{\Ab}^{(L+1)}=\mathbf{0}, \tilde{\Bb}^{(L+1)} = \mathbf{0}, \tilde{\Cb}^{(L+1)} = \mathbf{0} \text{ and } \tilde{\hb}_0^{(L+1)} = \mathbf{0}$ to have $\tilde{\hb}_t^{(L+1)} = \tilde{\hb}_t^{(L)} = \hb_t^{(L)} \forall t$ which implies $\tilde{h}=h$ and proves the inclusion $\hCPBIRNN{n,L} \subseteq \hCPBIRNN{n,L+1}$.

~\paragraph{Strict Inclusion: }
We show there exists a function $\tilde{h} \in \hCPBIRNN{n,L+1,R}$ that cannot be computed by any function $h \in \hCPBIRNN{n,L,R}$, that is $\hCPBIRNN{n,L,R} \not \supset \hCPBIRNN{n,L+1,R}$.
Let $poly(\xb,k)$ denote the set of functions having a polynomial dependency in $\xb$ of degree up to $k$, irrespective of dependencies in other variables. 
First, we prove by induction that for any $h \in \hCPBIRNN{n,L,R}$, the second time step output $\hb_2^{(L)}$ has a polynomial dependency in $\xb_1$ of degree at most $L$, i.e., $\hb_2^{(L)}\in poly(\xb_1,L)$~(using a slight abuse of notation):
\begin{itemize}[left=10pt]
    \item First, observe from Lemma \ref{lemma:induction-2RNN} that the hidden states in the first time step are given by
    \[ \hb_1^{(l)}  =\left( \prod_{i=1}^{l} ( \CP{\Ab,\Bb,\Cb}^{(i)} \times_{1} \hb_{0}^{(i)})^{\top} \right)\xb_{1} \]
    which means that $\hb_1^{(l)} \in poly(\xb_1,1)\; \forall l$.
    \item In the second time step,  Lemma \ref{lemma:induction-2RNN} can similarly be used to obtain
    \[
        \hb_2^{(l)} = \left( \prod_{i=1}^{l}(\CP{\Ab,\Bb,\Cb}^{(i)}\times_{1}\hb_{1}^{(i)})^{\top} \right) \xb_{2}.
    \]
    $\hb_2^{(l)}$ contains terms with a maximal number of $l$ factors in $\{\hb_1^{(i)}\}_{i=1}^{l}$, which, given  that $\hb_1^{(l)} \in poly(\xb_1,1)$, implies that $\hb_2^{(l)} \in poly(\xb_1,l)$.
\end{itemize}

Now we prove there exists a function $\tilde{h} \in \hCPBIRNN{n,L+1,R}$ such that $\tilde{\hb}_2^{(L+1)}=\diag(\xb_1)^{L+1}\xb_2$. We begin with the case $n=d=R$ and generalize after. Consider the function $\tilde{h}$ parameterized by $\{\tilde{\Ab}^{(l)} = \Ib, \tilde{\Bb}^{(l)} = \Ib, \tilde{\Cb}^{(l)} = \Ib, \tilde{\hb}_0^{(l)}=[1, \dots, 1]^\top \}_{l=1}^L$. From Lemma~\ref{lemma:induction-2RNN} and by observing that for any $\yb \in \Rbb^{n}$, $\CP{\tilde{\Ab}^{(l)}, \tilde{\Bb}^{(l)}, \tilde{\Cb}^{(l)}}  \times_1 \yb = \diag(\yb)$, we obtain:
\[ 
    \tilde{\hb}_t^{(l)}=\left(\prod_{i=1}^l \diag(\tilde{\hb}_{t-1}^{(i)}) \right) \xb_t.
\]
Given the hidden states initialization, we have $\diag(\tilde{\hb}_0^{(l)}) = \Ib \quad \forall l$ and the hidden states in the first time step become $\tilde{\hb}_1^{(l)} = \xb_1$. Inserting in the expression for $\tilde{\hb}_2^{(l)}$, we obtain
\[ 
    \tilde{\hb}_2^{(l)}= \diag(\xb_1)^l \xb_2.
\]
This proves that $\tilde{h}$ is such that $\tilde{\hb}_2^{(L+1)}=\diag(\xb_1)^{L+1}\xb_2$. This function can not be computed by any $h \in \hCPBIRNN{n, L,R}$ as for any such $h$, $\hb_2^{(L)}\in poly(\xb_1,L)$.

In the case $n>d$ and $R>d$, we set $\tilde{\Ab}^{(1)}_{ij} = \tilde{\Cb}^{(1)}_{ij} =\tilde{\Bb}^{(1)}_{ij} = \delta_{ij} \forall i,j \in [1,d]$, and 0 in the other rows and columns. The result will be the same, but with 0 padding on the extra dimensions.
In the case $n<d$ or $R<d$, let $k = \min(n,R)$, we set  $\tilde{\Ab}^{(1)}_{ij} = \tilde{\Cb}^{(1)}_{ij} =\tilde{\Bb}^{(1)}_{ij} = \delta_{ij} \forall i,j \in [1,k]$, and 0 in the other rows and columns. The result will be the same, but for the first $k$ dimensions of the inputs : $\hb_2^{(L+1)}=\diag([\xb_1]_{1:k})^{L+1}[\xb_2]_{1:k}.$
\end{proof}

\subsection{Proof of Theorem~\ref{thm:cprnn}}\label{apx:thm-cprnn}
\begin{theorem*}
    For $n>1$, $L\geq1$ and any $R\leq R_\textrm{max}$, $\hCPBIRNN{n,L,R} \subsetneq \hCPBIRNN{n,L,R+1}$ for linear CPBIRNNs if $R<n$.
\end{theorem*}

\begin{proof}
~\paragraph{Inclusion: }
We begin by showing that for any $h \in \hCPBIRNN{n,L,R}$, there exists a function $\tilde{h} \in \hCPBIRNN{n, L,R+1}$ such that $h=\tilde{h}$, that is $\hCPBIRNN{n,L,R} \subseteq \hCPBIRNN{n,L,R+1}$.
Let the function $h$ be computed by a CPBIRNN of rank $R$ parameterized by $\{\Ab^{(l)}, \Bb^{(l)}, \Cb^{(l)}, \hb_0^{(l)}\}_{l=1}^L$. For a CPBIRNN of rank $R+1$ computing $\tilde{h}$, we set:
\begin{align*}
\tilde{\Ab}_{ij}^{(l)}=\Ab^{(l)}_{ij} \quad \forall i \in [n], \; j\geq R\\
\tilde{\Bb}_{ij}^{(l)} = \Bb_{ij}^{(l)} \quad \forall i \in [d], j\geq R, \\
\tilde{\Cb}_{ij}^{(l)} = \Cb_{ij}^{(l)} \quad \forall i \in [n], j\geq R, \\
\quad \tilde{\hb}_0^{(l)} = \hb_0^{(l)},
\end{align*}
and all other dimensions of $\Ab^{(l)}, \Bb^{(l)} \text{ and } \Cb^{(l)}$ to 0, $\forall l \in [L]$. We thus have $\tilde{\hb}_t^{(L)} = \hb_t^{(L)} \forall t$ which implies $\tilde{h}=h$ and proves the inclusion $\hCPBIRNN{n,L,R} \subseteq \hCPBIRNN{n,L, R+1}$.
\end{proof}
~\paragraph{Strict Inclusion: }
To show strict inclusion, i.e. $\hCPBIRNN{n,L,R} \not \supset \hCPBIRNN{n,L,R+1}$, first observe that from Lemma~\ref{lemma:induction-2RNN} and Equation~\ref{eq:CP-prod},
 \[ 
 \hb_1^{(l)}  =\left( \prod_{i=1}^{l} ( \CP{\Ab,\Bb,\Cb}^{(i)} \times_{1} \hb_{0}^{(i)})^{\top} \right)\xb_{1} 
  = \left( \prod_{i=1}^{l}\left[ \C^{(i)} \diag({\A^{(i)}}^\top \hb_{t-1}^{(i)}) {\B^{(i)}}^\top \right] \right)\xb_{1}. 
 \]
 One can easily check that, independently of $l$, for any CPBIRNN of rank $R$, the image of the map $\xb_1 \mapsto \hb^{(L)}$ has dimension at most $R$, i.e. $\Im(h_1(\Rbb^d))\leq R$. To complete the proof for strict inclusion, we only have to show there exists a CPBIRNN of rank $R+1$ that reaches that limit ($\Im(h_1(\Rbb^d))=R+1$), since no CPBIRNNs of rank $R$ could compute this map. It suffices to consider any CPBIRNNs such that $\diag({\A^{(i)}}^\top \hb_{0}^{(l)}), \B^{(l)}, \C^{(l)}$ for $l\in [L]$ are full rank weight matrices.

\subsection{Proof of Theorem~\ref{thm:WFA}}\label{apx:thm-wfa}
Before presenting the proof of Theorem~\ref{thm:WFA}, we introduce a lemma that will be used in the proof of Theorem~\ref{thm:WFA}.

\begin{lemma}\label{apx:lemma_wfa_1layer}
Let $f$ be a function over sequences of $d$-dimensional vectors  computed by a 2RNN  with $n$ states ($n\geq d$) and parameters $\At, \Ub=\mathbf{0}, \Vb=\mathbf{0}, \bb = \mathbf{0}, \hb_0 = \fb_0 \neq \mathbf{0}$ with $\At$ generic.

Then, the function $f$ cannot be computed by any RNN with one layer of any width and nonlinear activations applied only in depth. 

\end{lemma}
\begin{proof}
    The computation of a 1 layer RNN of hidden size $\tilde{n}$ with nonlinear activations $\phi$ applied only in depth at each time step $t $ is given by  $\yb_t = \phi (\hb_t)$ where  $\hb_t= \Ub  \xb_{t} + \Vb \hb_{t-1} + \bb$. Here we assume $\phi$ is a homeomorphism as are most common activation functions ($\tanh$, sigmoid, ReLU on $\Rbb_{\geq 0})$.
    For such a RNN to \textit{compute} $f$, means there exists  a linear map $\ell \in \linearmaps{\tilde{n},n}$ such that $\fb_t =  \ell(\yb_t)$ for all $t$.
    
    Observe that $\hb_t = h(\xb_1, \dots, \xb_t)$ with $h \in \linearmaps{dt,\tilde{n}}$ a linear map. 
    
    The function $f$ is a multi-linear mapping of $(\xb_1, \dots, \xb_t)$. Given its recursive nature, it can be computed as a bilinear map between $\xb_t$ and $\fb_{t-1}$, i.e. $f: \Rbb^d \times \Rbb^n \to \Rbb^n$, with $f(\xb_t, \fb_{t-1}) = (\At \times_2 \x_t) \fb_{t-1}$, but computing $f_{t-1}$ involves a nonlinear operation between $(\xb_1, \dots, \xb_{t-1})$. Indeed, the entries of a matrix obtained by multiplying slices of a generic tensor (i.e. $\At$) are polynomial functions of the tensor's entries. Since a nonzero polynomial (over $\mathbb{R}$) vanishes only on a set of Lebesgue measure zero (unless it is identically zero) and $\At$ is generic, all  entries of this matrix are nonzero with probability 1.
    Thus, $\fb_{t-1}$ is a polynomial of $(\xb_1, \dots, \xb_{t-1})$: 
    \[\fb_{t-1} = \left(\sum_{i_1, \dots, i_{t-1}} \underbrace{\At_{:,i_{t-1},:}, \dots \At_{:,i_{1},:} }_{\neq 0}(\xb_{t-1})_{i_{t-1}} \dots  (\xb_{1})_{i_{1}} \right)\fb_0\]

    Since $\phi$ is a homeomorphism, $h$ must encode $(\xb_1, \dots, \xb_t)$ in order for $\fb_t = \ell (\yb_t)$ for all $t$, which can only be the case if $\tilde{n} \geq dt$. However, for any (arbitrarily large) $\tilde{n}$, there will always be a $t$ sufficiently large such that  $\tilde{n} < dt$, and therefore, no such RNN can compute $f$.

\end{proof}

\begin{theorem*}
There exists a function computed by a single-layer 2RNN that cannot be computed by any RNN of arbitrary (finite) depth and width with nonlinear activation applied only in depth. 
\end{theorem*}

\begin{proof}
    We consider a 1 layer 2RNN of hidden size $n$ with parameters $\At, \Ub=\mathbf{0}, \Vb=\mathbf{0}, \bb = \mathbf{0}, \hb_0 = \fb_0$ satisfying the conditions of Lemma~\ref{apx:lemma_wfa_1layer}, computing the function $f: (\Rbb^d)^n \to \Rbb^n$   mapping $(\xb_1, \dots, \xb_T) \mapsto (\fb_1, \dots, \fb_T)$ with the computation at each time step being $\fb_t = (\At \times_2 \x_t) \fb_{t-1}$. By unrolling the computation, we obtain $\fb_t = (\At \times_2 \x_t)(\At \times_2 \x_{t-1})\dots (\At \times_2 \x_1) \fb_0$. The function $f$ can be viewed as a multi-linear mapping of $(\xb_1, \dots, \xb_T)$, and given its recursive nature, it can also be computed as a bilinear map between $\xb_T$ and $\fb_{T-1}$, i.e. $f: \Rbb^d \times \Rbb^n \to \Rbb^n$, with $f(\xb_T, \fb_{T-1}) = (\At \times_2 \x_T) \fb_{T-1}$. More generally, we can compute $\fb_T = (\At \times_2 \x_T) \dots (\At \times_2 \x_{T-i+1}) \fb_{T-i}$, so $f : ((\Rbb^d)^i \times \Rbb^n)\to \Rbb^n$ is a multi-linear map of $(\xb_T,..., \xb_{T-i+1}, \fb_{T-i})$.

    Now consider a RNN of hidden size $\tilde{n}$, depth $L$ and nonlinear activations $\phi^{(l)}$ applied only in depth that computes at each time step $t \in [T]$ and layer $l \in [L]$:  $\yb_t^{(l)} = \phi^{(l)} (\hb_t^{(l)})$ where  $\hb_t^{(l)}= \Ub^{(l)}  \yb_{t}^{(l-1)} + \Vb^{(l)} \hb_{t-1}^{(l)} + \bb^{(l)}$. We assume $\phi^{(l)}$'s are homeomorphisms (as are the common activations $\tanh$, sigmoid and ReLU on $\Rbb_{\geq 0}$). 

    We will suppose this RNN is capable of computing $f$ for any sequence length $T$ and show that it leads to a contradiction. 

    We write an output $\yb_t^{(L)}$ that would achieve the computation as $\bar{\fb}_t \in \Rbb^{\tilde{n}}$ and suppose that there exists a linear map $\ell \in \linearmaps{\tilde{n},n}$ to complete the computation $\fb_t =  \ell(\bar{\fb}_t)$ for all $t \in [T]$.

    Assume $\yb_t^{(L)} = \phi^{(L)}( \hb_t^{(L)}) =  \bar{\fb}_t \; \forall t$. Then, either (1)  $\hb_t^{(L)}$ encodes the full sequence $(\xb_1, \dots, \xb_t)$ or (2) $\hb_t^{(L)}$ encodes the 2RNN state from some previous time step along with all subsequent inputs, i.e., $(\xb_{t}, \dots, \xb_{t-i+1}, \fb_{t-i})$ for some $i \in [t-1] $.
    In case (1), since $\phi^{(l)}$ is a homeomorphism at all layers $l$ and the dimension of the space of all possible input sequences of length $t$, $(\xb_1, \dots, \xb_t)$, is $dt$, if  $\hb_t^{(L)}$ encodes $(\xb_1, \dots, \xb_t)$, it is necessary that $\hb_t^{(L)}$ lies in a space of dimension at least $dt$.

    However, we can always find $t$ such that the dimension of $\hb_t^{(L)}$, $\tilde n < dt$, contradicting the dimension condition required for  $\hb_t^{(L)}$ to encode $(\xb_1, \dots, \xb_t)$. We thus have to turn to case (2):  $\hb_t^{(L)}$ must encode $(\xb_{t}, \dots, \xb_{t-i+1}, \fb_{t-i})$.
 
    Recall that $\hb_t^{(L)}$ is a linear map of $(\yb_1^{(L-1)}, \dots,\yb_t^{(L-1)} )$. The information $\fb_{t-i}$ is thus encoded in (at least) one of the inputs $(\yb_{t-i}^{(L-1)}, \dots,\yb_t^{(L-1)} )$. Without loss of generality, let $\yb_{t-j}^{(L-1)}$ with $j \leq i$ encode $\fb_{t-i}$. We have that $\yb_{t-j}^{(L-1)}$ encodes all the information of $(\xb_{t-j}, \dots \xb_{t-i}, \fb_{t-i})$. This implies that one can find $\tilde{\phi}^{(L-1)}$ such that $\yb_{t-j}^{(L-1)} = \bar{\fb}_{t-j}$, which in turn implies that a RNN could compute $\bar{\fb}_t$ with $L-1$ layers. The same argument can then be used to show that if it can be computed with $L-1$ layers, it can be done with $L-2$, and so on, until 1 layer. 
    
    However, by Lemma~\ref{apx:lemma_wfa_1layer}, the function $f$ cannot be computed by any 1-layer RNN with activation only in depth: a contradiction.  
\end{proof}

\section{Experimental details}
\subsection{Copy, copy-sinus and sinus tasks (synthetic experiments)}
Given inputs $\xb_t \in \mathbb{R}^5$ sampled from $\mathcal{N}(0,1)$, the copy task presented in Section~\ref{sec:copy} consists of predicting $\xb_{t-p}$ at each time step, with a lag set to $p=8$.
Similarly, the copy-sinus and sinus tasks consist in predicting $\sin(\omega \xb_{t-p})$, with frequency $\omega=3$ and lags $p=4$ and 0, respectively. 
The training sets contain 10,000 sequences of length 16; the validation and test sets each contain 200 sequences. 
Models were trained with a batch size of 128, learning rates between 0.001 and 0.002 and early stopping on the validation set with a patience of 400. 
For the copy task, we studied a linear RNNs, while for the copy-sinus and sinus tasks, we used a RNN with $\tanh$ activation applied only in depth. 
Weights were initialized from a random uniform distribution, $\mathcal{U}[-\frac{1}{\sqrt{n}},\frac{1}{\sqrt{n}}]$. Figures~\ref{fig:copy_real} and~\ref{fig:copy_sin} show the average values over 3–5 random seeds, with the standard deviation indicated by the shaded area.
All experiments were run using one GPU~(RTX8000, L40S or V100) with 32GB of memory. Generating one point took less than 30 minutes for the shortest experiments~(copy) and less than 1.5 hours for the longest~(copy-sinus).

\subsection{Parity (synthetic experiments)}
To produce inputs for the parity task, we take the signs of $\xb_t \in \mathbb{R}^5$ sampled from $\mathcal{N}(0,1)$, resulting in vectors of filled with $-1$s and 1s. The training set contains 10,000 sequences of length 20; the validation and test sets each contain 2000 sequences. The targets consist in element-wise products of the input sequences: $\yb_t = \xb_1... \xb_t$. 
Models were trained with a batch size of 128, a learning rate 0.001~(0.0005 for few runs) and early stopping with a patience of 400 on the validation set. 
Figure~\ref{fig:parity} displays the mean values computed over 3–5 random seeds, with the shaded region representing the standard deviation.
We studied RNNs with $\tanh$ activation applied either in depth or in recurrence, with weights were initialized from a uniform distribution, $\mathcal{U}[-\frac{1}{\sqrt{n}},\frac{1}{\sqrt{n}}]$. 
Experiments were run on a single GPU~(RTX8000, L40S or V100) with 32GB of memory, and it took between 10 minutes and 2 hours for each point. 

\subsection{Language modeling experiments - Tiny Shakespeare}
Language modeling experiments were conducted on the Tiny Shakespeare dataset~\cite{Karpathy2015} at the character level, using an embedding size of 107. The dataset was split into training, validation and test sets with a ratio of 0.8, 0.1 and 0.1, respectively. Models were trained on sequences of length 64, with early stopping based on validation performance and a patience of 200 epochs. 
If the training loss plateaued before showing any learning~(i.e., decreasing a reasonable amount), optimization was restarted. In Figure~\ref{fig:shakespeare}, we present the average values over 3–5 random seeds, with the standard deviation shown as a shaded area. All experiments were run on a single GPU (RTX8000, L40S or V100) with 32GB of memory, and training times ranged from 40 minutes to 6 hours. 

RNNs and CPRNNs were initialized with weights drawn from a uniform distribution, $\mathcal{U}[-\frac{1}{\sqrt{n}},\frac{1}{\sqrt{n}}]$, and had $\tanh$ activation in the recurrent connections. 
For RNNs, a batch size of 128 and a learning rate of 0.001 were set across all depths and hidden sizes.
For CPRNNs, two optimization regimes were adopted to ensure stability: one with a higher learning rate (0.0005 - 0.001) and a batch size of 128, and another with a lower learning~(0.0001) and batch size of 32. 
The former was used for all hidden sizes at $L=1$, as well as for $n=[64, 128, 256]$ at $L=2$ and $L=4$. The latter was required for the larger models with $n=[512,1024]$ at $L=2$ and $L=4$.
Similarly, the rank was set to twice the hidden size for all configurations at $L=1$, except for $n=1024$, and for $n=[64, 128]$ at $L=2$ and $L=4$. In all other cases, the rank was set equal to the hidden size. 
S4 models used the diagonal-plus-low-rank kernel~\cite{gu2022parameterization}, GeLU activation, input normalization, and a dropout rate of 0.1. These models were trained with a learning rate of 0.0001 and a batch size of 32.

\subsection{Long Range Arena Benchmark on S4}
All experiments on the Long Range Arena benchmark~\cite{tay2021long} were conducted using S4 models based on the official implementation from the original S4 paper~\cite{gu2021efficiently}. We used an RTX8000 GPU with 48GB of memory. Each run for the shortest experiment~(ListOps) took between 5 and 10 hours, while for Retrieval and Images it ranged between 15 to 35 hours. The longest, Pathfinder, required between 30 and 60 hours. 

All parameters, except for the depth $L$ and number of features $H$, were kept the same as the original settings for each tasks.  
The depth used in the original paper was $L=6$, for our purposes, we varied the depth from 2 to 8 and adjusted the number of features $H$ to maintain a constant parameter count~(see Table~\ref{tab:s4-lra-apx}). 

\begin{table}[h!]
    \caption{S4 parameters for Long Range Arena experiments: depth $L$, number of features $H$, and the resulting number of parameters. State dimension $N$ is fixed (4 for Retrieval and ListOps, 64 for Images and Pathfinder). Rows corresponding to $L=6$ reflect configurations from~\cite{gu2021efficiently}.}
    \label{tab:s4-lra-apx}
    \centering
    \begin{tabular}{llcccc}
        \toprule
        $L$ & Parameter & Retrieval & Images & ListOps & Pathfinder \\
        \midrule
        \multirow{2}{*}{6} 
            & $H$ & 256 & 512 & 256 & 256 \\
            & $\params$ & 808k & 3.6M & 808k & 1.3M \\
        \midrule
        \multirow{2}{*}{2} 
            & $H$ & 446 & 922 & 446 & 484 \\
            & $\params$ & 808k & 3.6M & 808k & 1.3M \\
        \midrule
        \multirow{2}{*}{3} 
            & $H$ & 364 & 745 & 364 & 385 \\
            & $\params$ & 810k & 3.6M & 810k & 1.3M \\
        \midrule
        \multirow{2}{*}{8} 
            & $H$ & 221 & 440 & 221 & 220 \\
            & $\params$ & 806k & 3.6M & 806k & 1.3M \\
        \bottomrule
    \end{tabular}

\end{table}